\documentclass[11pt, dvipsnames, DIV=12]{article}
\usepackage[english]{babel}

\usepackage[utf8]{inputenc}
\usepackage[T1]{fontenc}
\usepackage{subcaption}
\usepackage{booktabs}
\usepackage{multirow}
\usepackage{url}
\usepackage{tikz}
\usepackage{paralist}
\usepackage{ifthen}
\usepackage{amsmath}
\newcommand{\CC}[1][]{$\text{C\hspace{-.25ex}}^{_{_{_{++}}}}
\ifthenelse{\equal{#1}{}}{}{\text{\hspace{-.625ex}#1}}$}

\usepackage{bm}
\usepackage{hyperref}
\usepackage{amsmath}
\usepackage{amssymb}
\usepackage{amsthm}
\usepackage{amsfonts}
\usepackage{thmtools}		
\usepackage{mleftright}
\usepackage{stmaryrd}
\usepackage{nicefrac}
\usepackage{fancyref}
\usepackage{algorithm}
\usepackage{algorithmicx}
\usepackage[noend]{algpseudocode}

\let\originalleft\left
\let\originalright\right
\renewcommand{\left}{\mathopen{}\mathclose\bgroup\originalleft}
\renewcommand{\right}{\aftergroup\egroup\originalright}

\theoremstyle{definition}
\newtheorem{theorem}{Theorem}

\newtheorem{lemma}[theorem]{Lemma}

\usepackage{thm-restate}
\usepackage[mathic=true]{mathtools}
\usepackage{fixmath}
\usepackage{siunitx}

\usepackage{pifont}

\usepackage{enumitem}
\setlist[enumerate]{itemsep=0.2ex, topsep=0.5\topsep}
\setlist[description]{itemsep=0.2ex, topsep=0.5\topsep}
\setlist[itemize]{itemsep=0.2ex, topsep=0.5\topsep}

\makeatletter
\def\thmt@refnamewithcomma #1#2#3,#4,#5\@nil{%
\@xa\def\csname\thmt@envname #1utorefname\endcsname{#3}%
\ifcsname #2refname\endcsname
\csname #2refname\expandafter\endcsname\expandafter{\thmt@envname}{#3}{#4}%
\fi
}
\makeatother

\usepackage[capitalise,noabbrev]{cleveref}   

\usepackage{microtype}
\usepackage{ellipsis}

\usepackage[scaled=0.86]{helvet}
\usepackage{lmodern}

\usepackage{soul}
\usepackage[auth-lg]{authblk}

\newcommand{\mW}{\mathbf{W}}

\newcommand{\mb}{\mathbf{b}}
\newcommand{\mw}{\mathbf{w}}
\newcommand{\format}[1]{\mathtt{format{#1}}}

\newcommand{\tangh}{\mathsf{tanh}}

\newcommand{\lms}{\{\!\!\{}
\newcommand{\rms}{\}\!\!\}}
\newcommand{\balpha}{\boldsymbol \alpha}
\newcommand{\bsigma}{\boldsymbol \sigma}
\newcommand{\btheta}{\boldsymbol \theta}
\newcommand{\bTheta}{\boldsymbol \Theta}

\title{VC dimension of Graph Neural Networks with Pfaffian activation functions}
\author{Giuseppe Alessio D'Inverno, Monica Bianchini, Franco Scarselli}
\date{}
\affil{Department of Information Engineering and Mathematics \\ University of Siena\\ 
Via Roma 56, I--53100, Siena,  Italy}

\begin{document}
\maketitle
\begin{abstract}
Graph Neural Networks (GNNs) have emerged in recent years as a powerful tool to learn tasks across a wide range of graph domains in a data--driven fashion. Based on a message passing mechanism, GNNs have gained increasing popularity due to their intuitive formulation,
closely linked to the Weisfeiler--Lehman (WL) test for graph isomorphism, to which they were demonstrated to be equivalent\cite{morris2019weisfeiler, xu2018powerful}. From a theoretical point of view, GNNs have been shown to be universal approximators, and their generalization capability --- related to the Vapnik Chervonekis (VC) dimension \cite{scarselli2018vapnik} --- has recently been investigated for GNNs with piecewise polynomial activation functions \cite{morris2023wl}. 
The aim of our work is to extend this analysis on the VC dimension of GNNs to other commonly used activation functions, such as the sigmoid and hyperbolic tangent, using the framework of Pfaffian function theory. Bounds are provided with respect to the architecture parameters (depth, number of neurons, input size) as well as with respect to the number of colors resulting from the 1--WL test applied on the graph domain. The theoretical analysis is supported by a preliminary experimental study. 
\end{abstract}

\section{Introduction}\label{sec:introduction}

Since Deep Learning (DL) has become a fundamental tool in approaching real--life applications~\cite{rolnick2022tackling,fresca2020deep,lam2023learning,jumper2021highly}, the urgency of investigating its theoretical properties has become more evident. Neural networks were then progressively studied  analyzing, for example, their expressive power in terms of approximating classes of functions \cite{hornik1989multilayer,hornik1991approximation,hammer2000approximation,daubechies2022nonlinear} or showing their limitations in the imitation of neurocognitive tasks \cite{brugiapaglia2020generalizing,brugiapaglia2022invariance,d2023generalization}.
The \textit{generalization capability} of a learning model,  intended as the capacity of correctly performing a specific task on unseen data, %belonging to the same distribution, 
has always been a core aspect to evaluate the effectiveness of proposed architectures \cite{jacot2018neural,neyshabur2018towards}. Several metrics and/or methods have been proposed over the years to evaluate such capability \cite{koltchinskii2001rademacher,haussler2018probably}.
Among them, the \textit{Vapnik Chervonenkis (VC) dimension} \cite{vapnik1968uniform} is a metric that measures the capacity of a learning model to \textit{shatter} a set of data points, which means that it can always realize a perfect classifier for any binary labeling of the input data.  Intuitively, the greater the VC dimension of the learning model, the more it will fit the data on which it has been trained. However, as it has been shown in \cite{vapnik2006estimation}, a large VC dimension leads to poor
generalization, i.e. to a large difference between the error evaluated on the training and on the test set. 
%to overfitting, yielding more likely poor generalization capability. 
Therefore, it is important to establish the VC dimension of a model, especially with respect to its hyperparameters, in order to make it capable of generalizing on unseen data.

Graph Neural Networks (GNNs) \cite{scarselli2008graph, zhou2020graph} are machine learning
architectures capable of processing graphs that represent patterns (or part of patterns) along with their relationships. GNNs are among the most used deep learning models nowadays, given the impressive performance they have shown in tasks related to structured data \cite{liu2022introduction}.  A great effort has been dedicated to assess their expressive power, mainly related to the  study of the so--called Weisfeiler--Lehman (WL) test \cite{leman1968} and its variants \cite{morris2019weisfeiler,bodnar2021weisfeiler,bodnar2021weisfeilerb}.
Indeed, the standard WL algorithm, which checks whether two graphs are isomorphic by iteratively assigning colors to their nodes, has been proved to be equivalent to GNNs in terms of the capability of distinguishing graphs ~\cite{xu2018powerful}.  However, little is known about the generalization capabilities of GNNs. In \cite{scarselli2018vapnik},  bounds for the VC dimension have been provided for the original GNN model, namely the first model being introduced. Very recently \cite{morris2023wl}, bounds have been found also for a  large class of modern GNNs with piecewise polynomial activation functions. Nevertheless, message passing GNNs with other common activation functions, such as hyperbolic tangent, sigmoid and arctangent, still lack characterization in terms of VC dimension.

This work aims to fill this gap, providing new bounds for modern message passing GNNs 
with Pfaffian activation functions. \textit{Pfaffian functions} are a large class of differentiable maps,
which includes the above mentioned common activation functions, i.e, 
$\mathsf{tanh},\mathsf{logsig},\mathsf{atan}$, and, more generally, most of the function used in Engineering having continuous derivatives up to any order. Our main contributions are listed below.
\begin{itemize}
\item We provide upper bounds for message passing GNNs with Pfaffian activation functions with respect to the main hyperparameters, such as the feature dimension, the hidden feature size, the number of message passing layers implemented, and the total number of nodes in the entire training domain. To address this issue, we exploit theoretical results in the literature that link the theory of Pfaffian function with the characterization of the VC dimension of GNNs via topological analysis.
\item We also study the trend of the VC dimension w.r.t. the colors in the dataset obtained by running the WL test. Theoretical results suggest that the number of colors have an important effect on the GNN generalization capability. On the one hand, a large total number of colors in the training set improves generalization, since it increases the examples available for learning; on the other hand, a large number of colors in each graph raises the VC dimension and therefore increases the empirical risk value. 
\item Our theoretical findings are assessed by a preliminary experimental study; specifically, we evaluate the gap between the predictive performance on the training and test data. 
\end{itemize}

The manuscript is organized as follows. In Section \ref{sec:related_work}, we offer an overview of work related to the addressed topic. In Section \ref{sec:notation}, we introduce the main concepts and the notation used throughout the manuscript. In Section \ref{sec:main}, we state and discuss our main theoretical results. The preliminary experiments aimed at validating our theoretical results are described in Section \ref{sec:experiments}. Finally, in Section \ref{sec:discussion}, we draw some conclusions, also providing a brief discussion of open problems and future research directions.
\section{Related Work}\label{sec:related_work}
In this section we collect the main contributions present in the literature relating to the generalization ability of GNNs, the calculation of the VC dimension and the theory of Pfaffian functions.
\paragraph{Generalization bounds for GNNs ---}Several approaches have been exploited to give some insights on the generalization capabilities of GNNs. In \cite{garg2020generalization}, new bounds are provided on the \textit{Rademacher complexity} in binary classification tasks; the study is carried out by focusing on the computation trees of the nodes, which are tightly linked to the 1--WL test \cite{krebs2015universal} \cite{d2021new}. Similarly, in  \cite{esser2021learning}, generalization bounds for Graph Convolutional Networks (GCNs) are derived, based on the Trasductive Rademacher Complexity, which differs from the standard Rademacher Complexity by taking into account unobserved instances. In \cite{verma2019stability}, the stability, and consequently the generalization capabilities of GCNs, are proved to be dependent on the largest eigenvalue of the convolutional filter; therefore, to ensure a better generalization, such eigenvalue should be independent of the graph size. Under the lens of the PAC--learnability framework, the generalization bounds reported in \cite{garg2020generalization} have been improved in \cite{liao2020pac}, showing a tighter dependency on the maximum node degree and the spectral norm of the weights. This result aligns with the findings in \cite{verma2019stability}. In \cite{ju2023generalization}, 
sharper bounds on the GNN stability to noise are provided by investigating the correlation between attention and generalization. Specifically, GCNs and Graph Isomorphism Networks (GINs) are considered. The results show a link between the trace of the Hessian of the weight matrices and the stability of GNNs. 
A correlation between attention and generalization in GCNs and GINs is empirically investigated also in \cite{knyazev2019understanding}.

\paragraph{VC dimension ---}
Since it was first introduced in \cite{vapnik1968uniform}, the VC dimension has become a widespread metric to assess the generalization capabilities of neural networks. In \cite{vapnik1994measuring}, the VC dimension is proven to be tightly related to how the test error correlates, in probability, with the training error. Bounds on the VC dimension have been evaluated for many baseline architectures, such as Multi Layer Perceptrons (MLPs) \cite{sontag1998vc} \cite{bartlett2003vapnik}, Recurrent Neural Networks (RNNs) \cite{koiran1998vapnik} and Recursive Neural Networks \cite{scarselli2018vapnik}. In \cite{scarselli2018vapnik}, bounds on the VC dimension of the earliest GNN model with Pfaffian activation function are provided as well, while, in 
\cite{esser2021learning}, GCNs with linear and ReLU activation functions are considered. Our contribution extends such results to generic GNNs described by Eq.~\eqref{def:gnn_upd} and is
particularly related to the work in \cite{morris2023wl}, where bounds for the VC dimension of modern GNNs are studied, when the activation function is a piecewise linear polynomial function. Bounds are derived also in terms of the number of colors computed by the 1--WL test on the graph domain. However, aside from \cite{scarselli2018vapnik}, all the aforementioned works focus solely on specific GNN models with piecewise polynomials activation functions, not considering common activation functions as arctangent, hyperbolic tangent or sigmoid. 
\paragraph{Pfaffian functions ---}
Pfaffian functions have been first introduced in \cite{khovanskiui1991fewnomials} to extend Bezout's classic theorem, which states that the number of complex solutions of a set of polynomial equations can be estimated based on their degree. 
The theory of Pfaffian functions has been exploited initially in \cite{karpinski1997polynomial} to characterize the bounds of the VC dimension of neural networks. Similarly, in  \cite{scarselli2018vapnik}, the same approach is used to provide the aforementioned bounds. Pfaffian functions have also proven useful for providing insights into the topological complexity of neural networks and the impact of their depth \cite{bianchini2014complexity}.

\section{Notation and basic concepts}\label{sec:notation}
In this section we introduce the notation used throughout the paper and the main basic concepts necessary to understand its content.

\paragraph{Graphs ---}An \textit{unattributed graph} $G$ can be defined as a pair $ (V,E) $,
where $V$ is the (finite) set of \textit{nodes} %$V\subset \mathbb{N}$ is a finite the set of \textit{nodes} 
and $E \subseteq V \times V$ is the set of \textit{egdes} between nodes. A graph can be defined by its \textit{adjacency matrix} $\mathbf{A}$, where $A_{ij}= 1$ if $e_{ij} = (i,j) \in E$, otherwise $A_{ij}= 0$. The \textit{neighborhood} of a node $v$ is represented by $\mathsf{ne}(v) = \{ u \in V | (u,v) \in E \}$. A graph $G$ is said to be \textit{undirected} if it is assumed that $(v,u)=(u,v)$ (and therefore its adjacency matrix is symmetric), \textit{directed} otherwise. 
A graph is said to be \textit{node--attributed} or \textit{labeled} if there exists a map $\balpha: V \rightarrow \mathbb{R}^q$ that assigns to every $v\in V$ a \textit{node attribute} (or \textit{label}) $\balpha(v) \in \mathbb{R}^q$. In this case, the graph can be defined as a triple $(V,E,\balpha)$.
%A graph is said \textit{edge-attributed} if a map $\beta: E \rightarrow \mathbb{R}^e$ exists. In this paper we will deal only with node-attributed graphs.
%In the folowing, we assume to work with finite domains of finite graphs.

\paragraph{The 1--WL test ---}The \textit{1st order Weisfeiler--Lehman test} (briefly, the \textit{1--WL test}) is a test for graph isomorphism, based on the so--called \textit{color refinement} procedure. Given two graphs $G_1 = (V_1, E_1)$ and $G_2 = (V_2, E_2)$, in a finite graph domain $\mathcal{G}$, we perform the following steps.
\begin{itemize}
    \item At initialisation, we assign a color $c^{(0)}(v)$  to each node $v \in V_1 \cup V_2$. 
    %(in case of unattributed graphs the initialisation is uniform, while for attributed graphs it \fs{ depends on the node attributes}), where $\Sigma \subset \mathbb{N}$ is the set of all possible colors.
    Formally, in the case of attributed graphs, we can define the color initialisation as 
\begin{equation*}
    c^{(0)} (v)= \mathsf{HASH}_0 (\balpha(v)),
\end{equation*}
where $\mathsf{HASH}_0: \mathbb{R}^q \rightarrow \Sigma$ is a function that codes bijectively node attributes to colors. In case of unattributed graphs, the initialisation is uniform,
and each node $v$  gets the same color $c^{(0)} (v)$.
%$q=1$ and $\ell_v = 1 \;  \forall v \in V$, $\forall G=(V,E) \in \mathcal{G}$
\item For $t>0$, we update the color of each node in parallel on each graph by the  following  updating scheme 
\begin{equation*}
    c_v ^{(t)}=\mathsf{HASH}(c_v^{(t-1)},\lms c_u^{(t-1)}| u \in \mathsf{ne}[v] \rms ), \; \; \forall v \in V_1 \cup V_2\,,
\end{equation*}
where $\mathsf{HASH}: \Sigma \times \Sigma^* \rightarrow \Sigma$ is a function mapping bijectively a pair (color, color multiset) to a single color.
\end{itemize}

To test whether the two graphs $G_1$ and $G_2$ are isomorphic or not,
%The color refinement algorithm stops when the coloring partition of the nodes for each graph is stable, i.e. $c^{(t+1)} = c^{(t)}$. 
the set of colors of the nodes of $G_1$ and $G_2$ are compared step by step; 
if there exist an iteration $t$ such that the colors are different, namely
$c^{(t)}_{G_1}: = \lms c^{(t)} (v) \;|\; v \in V_1   \rms$  is different from $c^{(t)}_{G_2}$, the graphs are declared as non--isomorphic. When no difference is detected, the procedure halts as soon as the node 
partition defined by the colors becomes stable. It has been proven~\cite{kiefer2020weisfeiler} that
$|V|-1$ iterations are sufficient, and sometimes necessary, to complete the procedure.
Moreover, the color refinement procedure can be used also to test whether
two nodes are isomorphic or not. Intuitively, two nodes are  isomorphic if their neighborhoods (of any order)
are equal; such an isomorphism can be tested by comparing the node colors at 
any step of the 1--WL test. In \cite{krebs2015universal}, it has been proven that for node isomorphism up to $2\max(|V_1|,|V_2|)-1$ refinement steps may be required. 

We would like to mention two important results that demonstrate the equivalence between GNNs and the 1--WL test in terms of their expressive power. The first result was established in \cite{xu2018powerful} and characterizes the equivalence of GNNs and the 1--WL test on a graph--level task. This equivalence is based on GNNs with generic message passing layers that satisfy certain conditions.
Another characterization is due to \cite{morris2019weisfeiler} and states the equivalence on a node coloring level, referring to the particular model defined by Eq.~\eqref{eq:morrisGNN}.

\paragraph{Graph Neural Networks (GNNs) ---}Graph Neural Networks  are a class of machine learning models suitable for processing structured data in the form of graphs. At a high level, we can formalize a GNN as a function $\mathbf{g}: \mathcal{G} \rightarrow \mathbb{R}^r$, where $\mathcal{G}$ is a set of node--attributed graphs and $r$ is the dimension of the output, which depends on the type of task to be carried out; in our setting, we will assume that $r=1$.
Intuitively, a GNN learns how to represent the nodes of a graph by vectorial representations,  called \textit{hidden features}, giving an encoding of the information stored in the graph
The hidden feature $\mathbf{h}_v$ of a node $v$ is, at the begining set equal to node attributes,
i.e., $\mathbf{h}_v^{(0)}=\balpha(v)$. Then, the features are updated according to the following schema  
\begin{equation}\label{def:gnn_upd}
    \mathbf{h}_v^{(t+1)} = \mathsf{COMBINE}^{(t+1)}\big(  \mathbf{h}_v^{(t)}, \mathsf{AGGREGATE}^{(t+1)} (\lms \mathbf{h}_u^{(t)} | u \in \mathsf{ne}(v) \rms)   \big), 
\end{equation}
for all $v\in V$ and $t =0 , \dots L-1 $, where $\mathbf{h}_v^{(t)}$ is the hidden feature of node $v$ at time $t$, $L$ is the number of layers of the GNN and $\lms \cdot \rms $ denotes a multiset.
Here $\{ \mathsf{COMBINE}^{(t)} \}_{t=1,\dots, L}$ and $\{ \mathsf{AGGREGATE}^{(t)} \}_{t=1,\dots, L}$ are functions that can be defined by learning from examples. Popular GNN models like GraphSAGE \cite{hamilton2017inductive}, GCNs \cite{kipf2016semi}, and Graph Isomorphism Networks \cite{xu2018powerful} are based on this updating scheme. 
The output $o$ is produced by a $\mathsf{READOUT}$ function, which, in graph--focused tasks,
takes in input the features of all the nodes, i.e. $o=\mathsf{READOUT}(\lms \mathbf{h}_u^{(L)}| u\in V\rms )$,
while, in node--focused tasks, is calculated on each node, i.e., $o_v=\mathsf{READOUT}( \mathbf{h}_v^{(L)} ), \, \forall v \in V$.

For simplicity, in the following we will assume that $\mathsf{COMBINE}^{(1)}$ has $p_{\mathsf{comb}^{(1)}}$ parameters and for every $t=2, \dots, L$ the number of parameters of  $\mathsf{COMBINE}^{(t)} $ is the same, and we denote it as $p_{\mathsf{comb}}$. The same holds for $\mathsf{AGGREGATE}^{(t)} $ and $\mathsf{READOUT}$, with the number of parameters denoted respectively as $p_{\mathsf{agg}}$ and $p_{\mathsf{read}}$. Thus, the total number of parameters in a GNN defined as in Eq.~\eqref{def:gnn_upd} is $\Bar{p} = p_{\mathsf{comb}^{(1)}} + p_{\mathsf{agg}^{(1)}} + (L-1)(p_{\mathsf{comb}} + p_{\mathsf{agg}}) + p_{\mathsf{read}}$.

For our analysis, following \cite{morris2019weisfeiler}, we also consider a simpler computational framework, which
has been proven to match the expressive power of the Weisfeiler--Lehman test \cite{morris2019weisfeiler}, and
is general enough to be similar to many GNN models. In such a framework, the hidden feature $\mathbf{h}_v^{(t+1)} \in \mathbb{R}^d$ at the message passing iteration $t+1$
%, for $t=1, \dots, L-1$, 
is defined as
\begin{equation}\label{eq:morrisGNN}
    \mathbf{h}_v^{(t+1)} = \bsigma \big(\mW^{(t+1)}_{\text{comb}} \mathbf{h}^{(t)}_v + \mW^{(t+1)}_{\text{agg}}  \mathbf{h}^{(t)}_{\mathsf{ne}(v)} + \mb^{(t+1)} \big ), 
\end{equation}
where $\mathbf{h}^{(t)}_{\mathsf{ne}(v)} = \mathsf{POOL} \big (\lms \mathbf{h}^{(t)}_u | u \in \mathsf{ne}(v) \rms \big )$, $\bsigma: \mathbb{R}^d \rightarrow \mathbb{R}^d$ is an element--wise activation function, and  $\mathsf{POOL}$ is the aggregating operator on the features of neighboring nodes, 
%. The aggregating operator is a non-learnable function, such as sum, mean, or minimum, applied across the hidden features of neighbors.  We will take $\mathsf{POOL}$  and defined as
\begin{equation*}
    \mathsf{POOL} \big ( \lms \mathbf{h}^{(t)}_u | u \in \mathsf{ne}(v) \rms \big )  = \sum\limits_{u\in \mathsf{ne}(v)}\mathbf{h}_u^{(t)}.
\end{equation*}
\noindent
With respect to Eq.~\eqref{def:gnn_upd}, we have that $\mathsf{AGGREGATE}^{(t)}(\cdot) = \mathsf{POOL}(\cdot)$ $\forall t=1, \dots, L$, while $\mathsf{COMBINE}^{(t+1)}(\mathbf{h}_v, \mathbf{h}_{\mathsf{ne}(v)}) = \bsigma \big(\mW^{(t+1)}_{\text{comb}} \mathbf{h}_v + \mW^{(t+1)}_{\text{agg}}  \mathbf{h}_{\mathsf{ne}(v)} + \mb^{(t+1)} \big ). $
\noindent
In this case, the $\mathsf{READOUT}$ function for graph classification tasks can be defined as
\begin{equation}\label{eq:morrisREADOUT}
    \mathsf{READOUT}\Big( \lms \mathbf{h}_v^{(L)} \; | \; v \in V  \rms \Big):=  f \Big( \sum_{v\in V} \mw \mathbf{h}_{v}^{(L)} + b \Big).
\end{equation}
For each node, the hidden state is initialized as $\mathbf{h}_v^{(0)} = \balpha (v) \in \mathbb{R}^q$. 
The learnable parameters of the GNN can be summarized as $$\bTheta := (\mW^{(1)}_{\text{comb}}, \mW^{(1)}_{\text{agg}}, \mb^{(1)}, \mW^{(2)}_{\text{comb}}, \mW^{(2)}_{\text{agg}}, \mb^{(2)}, \dots,  \mW^{(L)}_{\text{comb}}, \mW^{(L)}_{\text{agg}}, \mb^{(L)}, \mw, b),$$ with $\mW^{(1)}_{\text{comb}}, \mW^{(1)}_{\text{agg}} \in \mathbb{R}^{d \times q}$, $\mW^{(t)}_{\text{comb}}, \mW^{(t)}_{\text{agg}} \in \mathbb{R}^{d\times d}$, for $t=2, \dots, L$, $\mb^{(t)} \in \mathbb{R}^{d \times 1}$, for $t=2, \dots, L$, $\mw \in \mathbb{R}^{1 \times d}$, and $b \in \mathbb{R}$.

\paragraph{VC dimension ---} The VC dimension is a measure of
complexity of a set of hypotheses,
%has been introduced in \cite{vapnik1968uniform}   and has bee as a quantification measure of the 
%, i.e., the set of possible model configurations given a fixed architecture.
which can be used to bound the empirical error of machine learning models.
Formally, a binary classifier $\mathcal{L}$ with parameters $\btheta$ is said to \textit{shatter} a set of patterns $\{\mathbf{x}_1, \dots, \mathbf{x}_n\} \subseteq \mathbb{R}^q$ if, for any binary labeling of the examples $\{y_i\}_{i=1,\dots,n}$, $y_i \in \{0,1\}$, there exists $\btheta$ s.t. the model $\mathcal{L}$ correctly classifies all the patterns, i.e. $\sum \limits_{i=0}^n |\mathcal{L}(\btheta,\mathbf{x}_i)-y_i| = 0$. 
The \textit{VC dimension} of the model $\mathcal{L}$ is the dimension of the largest set that $\mathcal{L}$ can shatter.

The VC dimension is linked with the generalization capability of machine learning models.
Actually, given % a binary classifier $L$, 
a training and a test set for the classifier  $\mathcal{L}$, whose patterns are i.i.d. samples extracted
from the same distribution, the VC dimension allows to compute a bound, in probability, for the difference between the training and test error
\cite{vapnik1971uniform}:
\begin{equation*}\label{eq:diff_bound}
    \text{Pr}\bigg( \text{E}_{\text{test}} \leq  \text{E}_{\text{training}} + \sqrt{\frac{1}{N}\bigg [ \mathsf{VCdim} \bigg(\log \bigg(\frac{2N}{\mathsf{VCdim}}\bigg)+1\bigg) - \log \bigg(\frac{\eta}{4}\bigg) \bigg]}\bigg) = 1-\eta
\end{equation*}
for any $\eta>0$, where $\text{E}_{\text{test}}$ is the test error, $\text{E}_{\text{training}}$ is the training error, $N$ is the size of the training dataset and $\mathsf{VCdim}$ is the VC dimension of $\mathcal{L}$. 

\paragraph{Pfaffian Functions ---}A Pfaffian chain of order $\ell\geq 0$ and degree $\alpha\geq 1$, in an open domain $U\subseteq\mathbb{R}^n$, is a sequence of  analytic functions $f_1,f_2,\ldots, f_\ell$ over $U$, satisfying the differential equations
$$
\textit{d}f_j(\mathbf{x})=\sum_{1\leq i\leq n} g_{ij}(\mathbf{x},f_1(\mathbf{x}),\ldots,f_j(\mathbf{x}))\textit{d}x_i, \,\,\, 1\leq j \leq\ell.
$$
Here,  $g_{ij}(\mathbf{x},y_1,\ldots,y_j)$ are polynomials in $\mathbf{x}\in U$ and $y_1,\ldots,y_j\in \mathbb{R}$ of degree not exceeding $\alpha$. A function
$f(\mathbf{x})=P(\mathbf{x},f_1(\mathbf{x}),\ldots,f_\ell(\mathbf{x}))$, where $P(\mathbf{x},y_1,\ldots,y_\ell)$ is a polynomial of degree not exceeding $\beta$, is called a \textit{Pfaffian function of format} $(\alpha,\beta,\ell)$.

Pfaffian maps are a large class of functions that include most of the functions with continuous derivatives used in practical applications \cite{khovanskiui1991fewnomials}. In particular, the arctangent $\mathsf{atan}$, the logistic sigmoid $\mathsf{logsig}$ and the hyperbolic tangent $\mathsf{tanh}$ are  Pfaffian functions, with format  $\format{(\mathsf{atan})}= (3,1,2)$, $\format{(\mathsf{logsig})}= (2,1,1)$, and  $\format{(\tangh)}= (2,1,1)$, respectively.

\section{Theoretical results}\label{sec:main}

In this section we report the main results  on the VC dimension of  GNNs with Pfaffian activation functions. The proofs can be found in Appendix \ref{Appendix:proofs}.

\subsection{Bounds based on the network hyperparameters} 
Our main result provides a bound on the VC dimension of  GNNs in which
$\mathsf{COMBINE}^{(t)}$, $\mathsf{AGGREGATE}^{(t)}$ and $\mathsf{READOUT}$ are Pfaffian functions. More precisely, we consider  a slightly more general version of the GNN model in Eq.~(\ref{def:gnn_upd}),  where the updating scheme is
\begin{equation}\label{def:gnn_upd_g}
    \mathbf{h}_v^{(t+1)} = \mathsf{COMBINE}^{(t+1)}\big(  \mathbf{h}_v^{(t)}, \mathsf{AGGREGATE}^{(t+1)} (\lms \mathbf{h}_u^{(t)} | u \in V \rms, A_v)   \big)\,, 
\end{equation}
and $A_v$ is the $v$--th column of the connectivity matrix, which represents 
the neighborhood of $v$. The advantage of the model in 
Eq.~(\ref{def:gnn_upd_g}) is that it makes explicit the dependence of $\mathsf{AGGREGATE}^{(t)}$  on the graph connectivity. Actually, here we want to underline what the inputs of
$\mathsf{AGGREGATE}^{(t)}$ are to clarify and make formally 
precise the assumptions that those functions are Pfaffian and have a given format.

%and GNNs with logistic sigmoid activation function 
Our result provides a bound on the VC dimension  w.r.t. the total number $\bar{p}$ of parameters, the number of  computation units $H$,  the number of layers $L$, the feature dimension $d$, the maximun number $N$ of nodes in a graph, and the attribute dimension $q$. Here, we assume that GNN computation units include the neurons
computing the hidden features of each node and the outputs. Therefore, there is  a computation unit for each component of a feature, each layer, each node of the input graph  
and a further computation unit for the $\mathsf{READOUT}$.

\begin{theorem} \label{th:main_general}
Let us consider the GNN model described by Eq.~\eqref{def:gnn_upd}. If  $\mathsf{COMBINE}^{(t)}$, $\mathsf{AGGREGATE}^{(t)}$ and $\mathsf{READOUT}$ are Pfaffian functions with format $(\alpha_{\mathsf{comb}}, \beta_{\mathsf{comb}}, \ell_{\mathsf{comb}})$, $(\alpha_{\mathsf{agg}}, \beta_{\mathsf{agg}}, \ell_{\mathsf{agg}})$, $(\alpha_{\mathsf{read}}, \beta_{\mathsf{read}}, \ell_{\mathsf{read}})$, respectively, then
the VC dimension satisfies
\begin{equation}
\label{VCGNN}
\mathsf{VCdim}\bigl(\mathsf{GNN}\bigr) \leq  2\log B + \Bar{p}(16 +2 \log \bar{s} )
\end{equation}
  where $B\leq  2^{\frac{\bar{\ell}(\bar{\ell}-1)}{2}+1}(\Bar{\alpha} + 2 \Bar{\beta} -1)^{\Bar{p}-1} ((2 \Bar{p}-1)(\Bar{\alpha} + \Bar{\beta})- 2\Bar{p}+2 )^{\Bar{\ell}}$, $\Bar{\alpha}=\max \{ \alpha_{\mathsf{agg}} + \beta_{\mathsf{agg}} -1+\alpha_{\mathsf{comb}} \beta_{\mathsf{agg}},  \alpha_{\mathsf{read}}\}$, $\bar{\beta}=\max \{ \beta_{\mathsf{comb}}, \beta_{\mathsf{read}}\}$,  \\$\bar{p}=  p_{\mathsf{comb}^{(0)}} + p_{\mathsf{agg}^{(0)}} + (L-1)(p_{\mathsf{comb}} + p_{\mathsf{agg}}) + p_{\mathsf{read}} $,   $\bar{\ell}=\Bar{p}H$,  $H = LNd(\ell_{\mathsf{comb}} + \ell_{\mathsf{agg}} )+ \ell_{\mathsf{read}}$ and $\bar{s}=LNd+Nq+1$ hold. 
By substituting  the definitions in Eq.~\eqref{VCGNN}, we obtain
\small
 \begin{align}
\mathsf{VCdim}\bigl(\mathsf{GNN}\bigr) &\leq   \bar{p}^2 (LNd(\ell_{\mathsf{comb}} + \ell_{\mathsf{agg}} )+ \ell_{\mathsf{read}})^2 \nonumber \\
   &+ 2\bar{p} \log \left(3\gamma \right) \nonumber \\
  &+ 2\bar{p}\log\left((4\gamma \text{\hspace{-0.1cm}}-\text{\hspace{-0.1cm}}2)\bar{p})\text{\hspace{-0.1cm}}+\text{\hspace{-0.1cm}}2
  \text{\hspace{-0.1cm}}-\text{\hspace{-0.1cm}}2\gamma\right) \nonumber \\
  &+ \bar{p}(16 +2 \log(LNd+Nq+1 )) \label{VCGNN_ext}
  \end{align}
\normalsize
where $\Bar{\alpha}, \Bar{\beta} \leq \gamma$ for a constant $\gamma \in \mathbb{R}$. 
\end{theorem}
  
By inspecting the bound, we observe that the dominant term is $\bar{p}^2H^2 = \bar{p}^2(LNd(\ell_{\mathsf{comb}} + \ell_{\mathsf{agg}} )+ \ell_{\mathsf{read}})^2 $. Thus, Theorem \ref{th:main_general} suggests that the VC dimension is 
$O(\bar{p}^2L^2N^2d^2)$, w.r.t.  the  number of parameters $\bar{p}$ of the GNN, the number
of layers $L$, the number $N$ of  graph nodes, and the feature dimension $d$. Notice that those hyperparameters are related by constraints, which should be considered in order to understand how the VC dimension depends on each of them. Therefore, the VC dimension is at most $O(p^4)$ since, as the number $p$ of parameters grows, the number of layers $L$ and/or the feature dimension $d$ also increases. 

Interestingly, such a result is similar to those already obtained for feedforward and recurrent neural networks with Pfaffian activation functions.  Table~\ref{VCDcomparison} compares our result
with those available in the literature, highlighting that, even if GNNs have a more complex structure, the growth rate of the VC dimension, depending on the hyperparameters, is the same as the simpler models.

The following theorem provides more details and clarifies how the VC dimension depends on each hyperparameter.

\begin{theorem} \label{th:order_Pfaff_general}
Let $\mathsf{COMBINE}^{(t)}$, $\mathsf{AGGREGATE}^{(t)}$ and $\mathsf{READOUT}$ be the Pfaffian functions defined in Theorem \ref{th:main_general}. If $p_{\mathsf{comb}}, p_{\mathsf{aggr}}, p_{\mathsf{read}} \in \mathcal{O}(d)$, then the VC dimension of a GNN defined as in Eq.~\eqref{def:gnn_upd},  w.r.t. $\bar{p},N,L,d,q$ satisfies
\begin{align*}
	    \mathsf{VCdim}\bigl(\mathsf{GNN}\bigr)&\leq \mathcal{O}(\bar{p}^4 )\\
		\mathsf{VCdim}\bigl(\mathsf{GNN}\bigr)&\leq \mathcal{O}(N^2 )\\
		\mathsf{VCdim}\bigl(\mathsf{GNN}\bigr)&\leq \mathcal{O}(L^4 )\\
		\mathsf{VCdim}\bigl(\mathsf{GNN}\bigr)&\leq \mathcal{O}(d^6 )\\
		\mathsf{VCdim}\bigl(\mathsf{GNN}\bigr)&\leq \mathcal{O}(q^2 ) 
\end{align*}
\qed
\end{theorem}
\begin{table}[!tbp]
\centerline{
%\begin{tabular}{|l|c|c|}\hline
\begin{tabular}{|@{\,}l@{\,}|c|@{}c@{}|}\hline
{\bf Activation function} &  {\bf Bound} & {\bf References}  \\ \hline\hline
\multicolumn{3}{|c|}{Modern GNNs} \\ \hline
Piecewise polynomial   & %$O(Lp \log(\mathcal{C}p ) + L^2 p \log(\delta))$ 
$O(p \log(\mathcal{C}p ) + p \log(N))$ & \cite{morris2023wl} \\ \hline
$\mathsf{tanh}$, $\mathsf{logsig}$ or $\mathsf{atan}$ & $O(p^4N^2)$ & \textbf{this work} \\  \hline
$\mathsf{tanh}$, $\mathsf{logsig}$ or $\mathsf{atan}$ & $O(p^4\mathcal{C}^2)$ & \textbf{this work} \\  \hline\hline
\multicolumn{3}{|c|}{Original GNNs \cite{scarselli2008graph}} \\ \hline
Polynomial  &  $O(p\log(N))$  &  \cite{scarselli2018vapnik}\\ \hline
Piecewise polynomial   & $O(p^2N\log(N))$  & \cite{scarselli2018vapnik} \\ \hline
$\mathsf{tanh}$, $\mathsf{logsig}$ or $\mathsf{atan}$ & $O(p^4N^2)$ & \cite{scarselli2018vapnik} \\  \hline\hline
\multicolumn{3}{|c|}{Positional  RecNNs} \\ \hline
Polynomial  & $O(pN)$ &  \cite{hammer2001generalization} \\ \hline
$\mathsf{logsig}$ & $O(p^4N^2)$ &   \cite{hammer2001generalization}\\ \hline\hline
\multicolumn{3}{|c|}{Recurrent Neural Networks} \\ \hline
Polynomial  & $O(pN)$ &  \cite{koiran1997neural} \\ \hline
Piecewise polynomial  & $O(p^2N)$ &    \cite{koiran1997neural} \\ \hline
$\mathsf{tanh}$ or $\mathsf{logsig}$ & $O(p^4N^2)$ & \cite{koiran1997neural}\\ \hline\hline
\multicolumn{3}{|c|}{Multilayer Networks} \\ \hline
Binary  & $O(p \log p)$ & \cite{baum1988size,maass1994neural,sakurai1995vc}  \\ \hline
Polynomial & $O(p\log p)$  & \ \cite{goldberg1993bounding} \\ \hline
Piecewise polynomial  & $O(p^2)$  & \cite{goldberg1993bounding,koiran1997neural}  \\ 
\hline
$\mathsf{tanh}$, $\mathsf{logsig}$ or $\mathsf{atan}$  & $O(p^4)$ &\cite{karpinski1997polynomial} \\ \hline 
%\hline
\end{tabular}}
\caption{Upper bounds on the VC dimension of common architectures, where
$p$ is the  number of network parameters,
$N$ the number of nodes in the input graph or sequence, and $\mathcal{C}$ the maximum number of colors per graph.
}\label{VCDcomparison}
\vspace*{-3mm}
\end{table}

The proof of Theorem \ref{th:main_general} adopts the same reasoning used
in \cite{karpinski1997polynomial}  to derive a bound on the VC dimension
of  feedforward neural networks with Pfaffian activation functions, and used
in \cite{scarselli2018vapnik} to provide a bound for the first GNN model.
Intuitively, the proof is based on the following steps: it is shown that
the computation of the GNNs on graphs can be represented by a set of equations
defined by Pfaffian functions with format $(\bar{\alpha},\bar{\beta},\bar{\ell})$, where 
$\bar{\alpha},\bar{\beta},\bar{\ell}$ are those defined in the theorem;
then, the bound is obtained exploiting a result in \cite{karpinski1997polynomial} that 
associates the  VC dimension to the number of connected components in the inverse image of a system of Pfaffian equations. Finally, a result in~\cite{gabrielov2004complexity}  allows to estimate the required number of connected components.  Note that our bound and other bounds obtained for networks with Pfaffian activation functions are larger than those
for networks with simpler activations. As explained in \cite{karpinski1997polynomial} \cite{scarselli2018vapnik},
such a difference is likely due to the current limitations of mathematics in this field, which makes tight bounds more difficult to achieve with Pfaffian functions.

%In order to provide a the a specific GNN model,
We now specifically derive bounds for the VC dimension for the architecture described by Eqs.~\eqref{eq:morrisGNN}--\eqref{eq:morrisREADOUT}. 

\begin{theorem} \label{th:main}
Let us consider the GNN model described by Eqs.~\eqref{eq:morrisGNN}--\eqref{eq:morrisREADOUT}. If  $\sigma$ is a Pfaffian function in $\mathbf{x}$ with format $(\alpha_\sigma,\beta_\sigma,\ell_\sigma)$, then
the VC dimension satisfies
 \begin{align*}
\mathsf{VCdim}\bigl(\mathsf{GNN}\bigr) &\leq  2\log B + \Bar{p}(16 +2 \log \bar{s} ),
  \end{align*}
  where $B\leq  2^{\frac{\bar{\ell}(\bar{\ell}-1)}{2}+1}(\Bar{\alpha} + 2 \Bar{\beta} -1)^{\Bar{p}-1} ((2 \Bar{p}-1)(\Bar{\alpha} + \Bar{\beta})- 2\Bar{p}+2 )^{\Bar{\ell}}$, $\Bar{\alpha}=2+3\alpha_\sigma$, $\bar{\beta}=\beta_\sigma$, $\bar{\ell}=\Bar{p}H\ell_\sigma$, and $\bar{s}=LNd+Nq+1$ hold.
  
\noindent 
In particular, if $\sigma$ is the logistic sigmoid activation function, we have
\begin{equation*}
\mathsf{VCdim}\bigl(\mathsf{GNN}\bigr) \leq   \Bar{p}^2H^2+ 
  2{\Bar{p}} \log \left(9\right) + 2\Bar{p}H\log\left(16\Bar{p}\right) +\Bar{p}(16 +2 \log(\Bar{s}) ). 
  \end{equation*}
  \qed
\end{theorem}

The proof of Theorem \ref{th:main} can be found in Appendix \ref{Appendix:proofs}. 

Interestingly, the bounds on the VC dimension that can be derived from Theorem \ref{th:main}, w.r.t the hyperparameters, turn out to be the same derived in Theorem \ref{th:order_Pfaff_general}. Thus, even if the considered model is simpler, those bounds do not change.

% The following theorem shows that the obtained bounds are perfectly coherent with the ones achieved in the general case.

% \begin{theorem} \label{th:order_Pfaff}
% The VC dimension of a GNN defined as in Eqs. \eqref{eq:morrisGNN},\eqref{eq:morrisREADOUT}., w.r.t. $\bar{p},H,N,L,d,q$ satisfies
% \begin{align*}
% 	    \mathsf{VCdim}\bigl(\mathsf{GNN}\bigr)&\leq \mathcal{O}(\bar{p}^4 )\\
% 		\mathsf{VCdim}\bigl(\mathsf{GNN}\bigr)&\leq \mathcal{O}(N^2 )\\
% 		\mathsf{VCdim}\bigl(\mathsf{GNN}\bigr)&\leq \mathcal{O}(L^4 )\\
% 		\mathsf{VCdim}\bigl(\mathsf{GNN}\bigr)&\leq \mathcal{O}(d^6 )\\
% 		\mathsf{VCdim}\bigl(\mathsf{GNN}\bigr)&\leq \mathcal{O}(q^2 )
% \end{align*}
% \end{theorem}

\subsection{Bounds based on the number of the 1--WL colors}\label{subsec:colors}
The developed theory is also easily applied to the case when nodes are grouped according to their colors defined by the Weisfeiler--Lehman algorithm.  Intuitively, since GNNs produce the same features  on group of nodes with the same color,  the  computation can be 
simplified by considering each group as a single entity. As consequence, 
the bounds on VC dimension can be tightened by using colors in place of nodes. 
Formally, for a given graph $G$, let $C_1(G)=\sum_{i=1}^T C^t(G)$
be the number of colors generated by the 1--WL test, where $ C^t(G)$ is the  number of colors at step $t>0$. Moreover, 
let us assume that  $ C^t(G)$ is bounded,  namely 
there exists $C_1$ such that $C_1(G)\leq C_1$ for all the graphs $G$ in the domain $\mathcal{G}$. The following theorem provides a bound on the VC dimension w.r.t. the number of colors produced by the 1--WL test.
\begin{theorem}\label{th:vcdim_colors}
Let us consider the GNN model described by Eqs.~\eqref{eq:morrisGNN}--\eqref{eq:morrisREADOUT}
using the logistic sigmoid $\mathsf{logsig}$ as the activation function.
	Assume a subset $\mathcal{S} \subseteq \mathcal{G}$. The VC dimension of the GNN satisfies   
\begin{align*}
		\mathsf{VCdim}\bigl(\mathsf{GNN}(C_1)\bigr)&\leq \mathcal{O}(C_1^2 )\\
		\mathsf{VCdim}\bigl(\mathsf{GNN}(C_0)\bigr)&\leq \mathcal{O}(\log(C_0) )
\end{align*}
\qed
\end{theorem}
The theorem suggests that the VC dimension depends quadratically on the  total number of node colors and logarithmically on the initial number of colors. Actually, a GNN  processes
all the nodes of a graph at the same time and the GNN architecture is similar
to a feedforward network where some  computation units are replicated at each node. Thus, the complexity of the GNN grows with the number of nodes and this explains the dependence of the VC dimension on the number of nodes (see Theorem~\ref{th:order_Pfaff_general}). On the other hand, nodes with the same colors
cannot be distinguished by the GNN: this means that, in theory, we can use the same computation units for a group of nodes sharing the color. Therefore, to get tighter bounds on the VC dimension, we can consider the number of colors in place of the number of nodes. 

Finally, it is worth mentioning that, the presented theorems suggest that 
GNNs may have a worst generation capability when the domain is composed by graphs with may different colors. This happens because when the number of colors in each graph increases, the  VC dimension increases as well.  On the other hand, the generalization capability  benefits from  a large total number  of colors in the training set. Actually, generalization depends not only on VC dimension, but, obviously,  also on the number of patterns in  training set (see Eq.~\eqref{eq:diff_bound}). In GNN graph--focused tasks, graphs play the role of
 patterns, where we count only the graphs with different colors, as those with the same colors are just copies of the same pattern. A similar reasoning applies to node--focused task, by counting the total number of nodes with different color in training set.

%The proof of this theorem can be found in Appendix \ref{Appendix:proofs}.
\section{Experimental validation}\label{sec:experiments}

In this section, we present an experimental validation of our theoretical results. We will show how the VC dimension of GNNs, descibed in Eqs.~\eqref{eq:morrisGNN}--\eqref{eq:morrisREADOUT}, changes as the hyperparameters vary, respecting the bounds found in Theorems \ref{th:order_Pfaff_general} and \ref{th:vcdim_colors}.

\subsection{Experimental setting}
We design two experiments to assess the validity, respectively, of Theorems \ref{th:order_Pfaff_general} and \ref{th:vcdim_colors}. In both cases, we train a Graph Neural Network, composed by message passing layers, defined as in Eq.~\eqref{eq:morrisGNN}, where the activation function $\bsigma$ is $\mathsf{arctan}$ or $\mathsf{tanh}$; the final $\mathsf{READOUT}$ layer is an affine layer with $\mathbf{W}_{\mathsf{out}} \in \mathbb{R}^{1 \times \mathsf{hd}}$, after which a $\mathsf{logsig}$ activation function is applied.  The model is trained via Adam optimizer with an initial learning rate $\lambda = 10^{-3}$. The hidden feature size is denoted by $\mathsf{hd}$ and the number of layers by $l$.
\begin{itemize}
    \item[\textbf{E1}:] We measure the evolution of the difference between the training accuracy and the validation accuracy, $\mathsf{diff}= \mathsf{training}\_\mathsf{acc} - \mathsf{test\_acc}$, through the training epochs, over three different datasets taken from the TUDataset repository \cite{morris2020tudataset}. In particular, \textbf{PROTEINS} \cite{borgwardt2005protein} is a dataset of proteins represented as graphs which contains both enzymes and non--enzymes; \textbf{NCI1} \cite{wale2008comparison} is a dataset of molecules relative to anti--cancer screens where the chemicals are assessed as positive or negative to cell lung cancer; finally, \textbf{PTC-MR} \cite{helma2001predictive} is a collection of chemical compounds represented as graphs which report the carcinogenicity for rats. The choice of the datasets has been driven by their binary classification nature. Their statistics are summarized in Table \ref{tab:benchmark:statistics}. In the experiments, firstly, we fix the hidden feature size to $\mathsf{hd} = 32$ and let the number of layers vary in the range $l \in [2,3,4,5,6]$, to measure how $\mathsf{diff}$ evolves. Subsequently,  we fix the number of layers to $l = 3$ and let the hidden feature size vary in the range,  $\mathsf{hd} \in [8,16,32,64,128]$, to perform the same task. We train the model for 500 epochs in each run, with the batch size set to 32.
    \item[\textbf{E2}:] We measure the evolution of the difference between the training accuracy and the validation accuracy, $\mathsf{diff}$, through the training epochs over the dataset \textbf{NCI1}, whose graphs are increasingly ordered according to the ratio $\frac{|V(G)|}{C^t(G)}$ and split in four different groups. The intuition here is that, being the number of graph nodes bounded, splitting the ordered dataset as described above, should provide four datasets in which the total number of colors is progressively increasing. The hidden size is fixed at $\mathsf{hd}=16$, the number of layers is $l=4$, the batch size is fixed equal to 32. In Table \ref{table:summary}, we report the reference values for each split. We train the model for 2000 epochs (the number of epochs is greater than in the \textbf{E1} task because \textbf{E2} shows greater instability during training).
\end{itemize}

\begin{table}[ht]
    \centering
    \begin{tabular}{|c|c|c|c|c|}
    \hline
        Dataset & \# Graphs & \# Classes & Avg. \# Nodes & Avg. \# Edges \\
        \hline \hline
       \textbf{\small{PROTEINS}}  & 1113 & 2 & 39.06 & 72.82 \\
        \textbf{\small{NCI1}}  &  4110 & 2 & 29.87 & 32.30\\
         \textbf{\small{PTC-MR}} & 344 & 2 & 14.29 & 14.69 \\ \hline
    \end{tabular}
    \caption{Statistics on the benchmark datasets used for \textbf{E1}.}
    \label{tab:benchmark:statistics}
\end{table}

\begin{table}[ht]
\centering
\begin{tabular}{|c||c|c|c|c|}
\hline
   & Split 1 & Split 2 & Split 3 & Split 4 \\
     \hline \hline
  \# Nodes   & 27667 & 30591 & 31763 & 32673\\ [0.5ex] 
\# Colors  & 26243 & 26569 & 24489 & 16348 \\
$\min \limits_G\frac{\#\text{Nodes}(G)}{\#\text{Colors}(G)}$    & 1.000 & 1.105 & 1.208 &1.437\\
$\max \limits_G\frac{\#\text{Nodes}(G)}{\#\text{Colors}(G)}$    & 1.105 & 1.208 &1.437 & 8\\
     \hline
\end{tabular}\caption{Summary of the parameters for each split of the ordered \textbf{NCI1} dataset in task \textbf{E2}.}\label{table:summary}
\end{table}

Each experiment is statistically evaluated over $10$ runs. The overall training is  performed on an Intel(R) Core(TM) i7-9800X processor running at 3.80GHz, using 31GB of RAM and a GeForce GTX 1080 Ti GPU unit.
The code developed to run the experiments exploits the Python package \url{Pytorch Geometric} and can be found at \url{https://github.com/AleDinve/vc-dim-gnn}.

\subsection{Experimental results}

\paragraph{Task \textbf{E1} ---}Numerical results for the \textbf{NCI1} dataset are reported in \cref{fig:atan_NCI1,fig:tanh_NCI1}, for $\mathsf{arctan}$ and $\mathsf{tanh}$ activations, respectively, while results for \textbf{PROTEINS} and \textbf{PTC\_MR} are reported in Appendix \ref{Appendix:experiments}. In particular, in both \cref{fig:atan_NCI1,fig:tanh_NCI1}, the evolution of $\mathsf{diff}$ is shown as the number of epochs varies, for different
values of $\mathsf{hd}$ keeping fixed $l = 3$ in (a), and for different values of $l$ keeping fixed
$\mathsf{hd} = 32$ in (c), respectively. Moreover, the evolution of $\mathsf{diff}$ as the hidden size increases, for
varying epochs, is shown in (b), while (d) depicts how $\mathsf{diff}$ evolves as the number of layers increases.

\begin{figure}[h!]
    \begin{subfigure}{0.49\linewidth}
        \includegraphics[width=\textwidth]{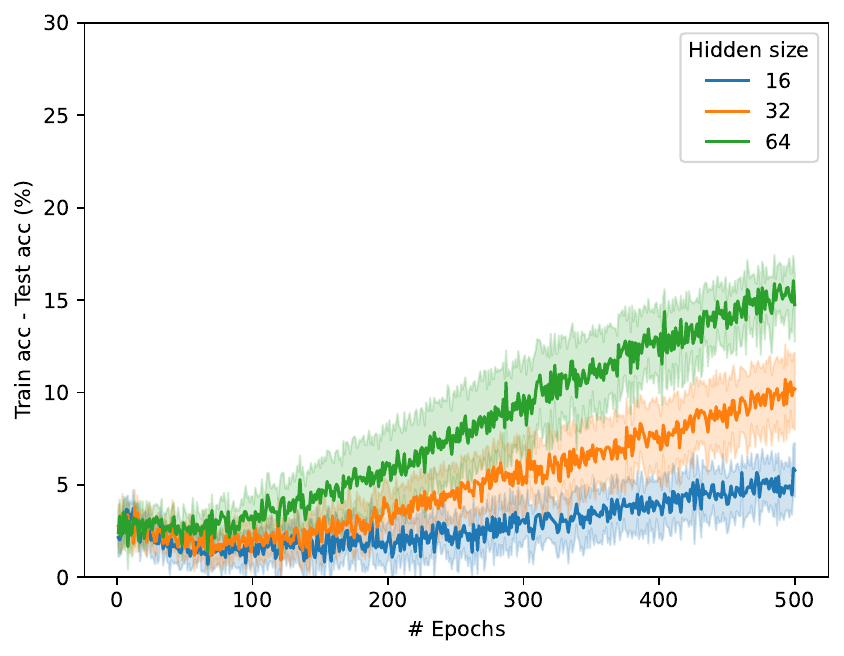}
        \caption{}
        \label{fig:atan_NCI1_a}
    \end{subfigure}
    \begin{subfigure}{0.49\linewidth}
        \includegraphics[width=\textwidth]{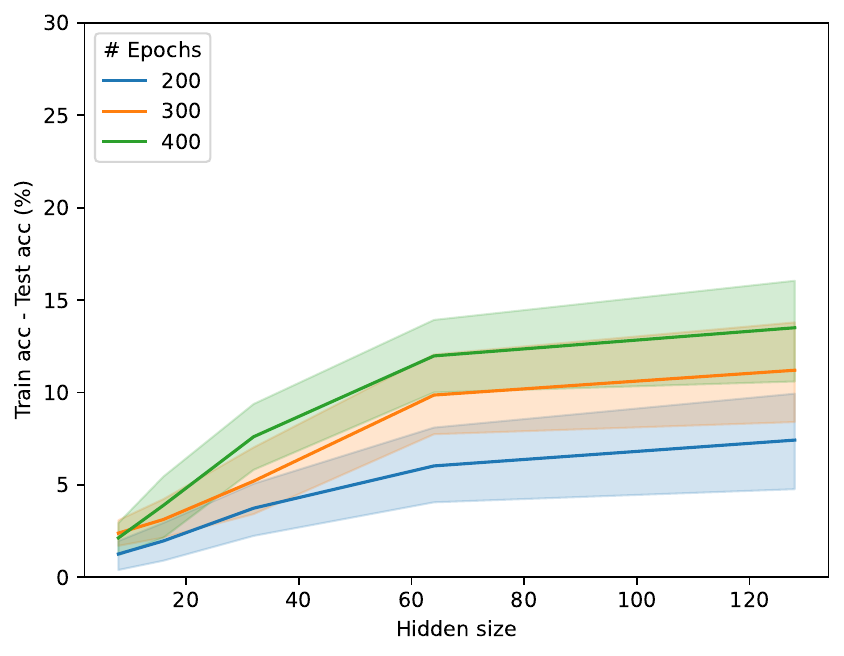}
        \caption{}
        \label{fig:atan_NCI1_b}
    \end{subfigure} 
    \begin{subfigure}{0.49\linewidth}
        \includegraphics[width=\textwidth]{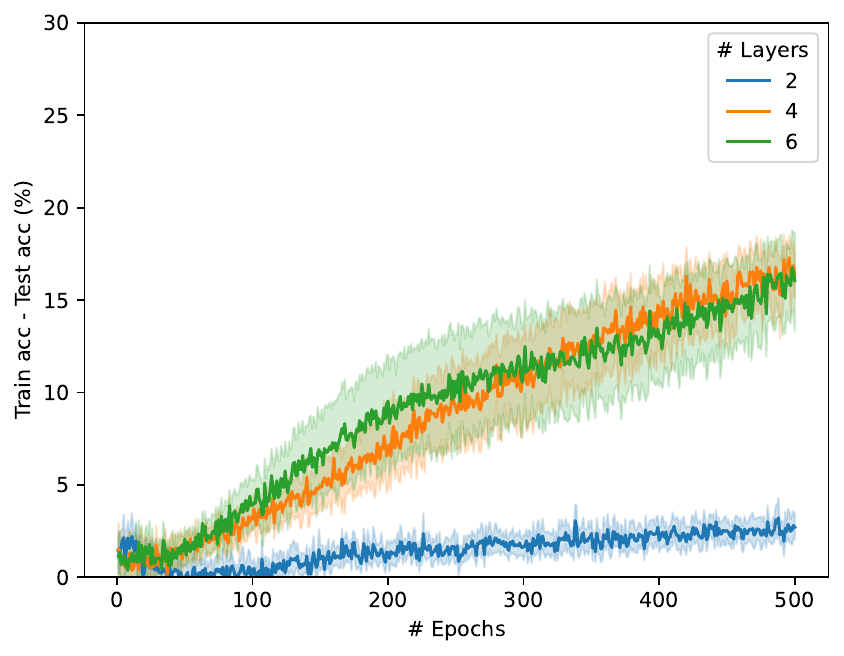}
        \caption{}
        \label{fig:atan_NCI1_c}
    \end{subfigure}
    \begin{subfigure}{0.49\linewidth}
        \includegraphics[width=\textwidth]{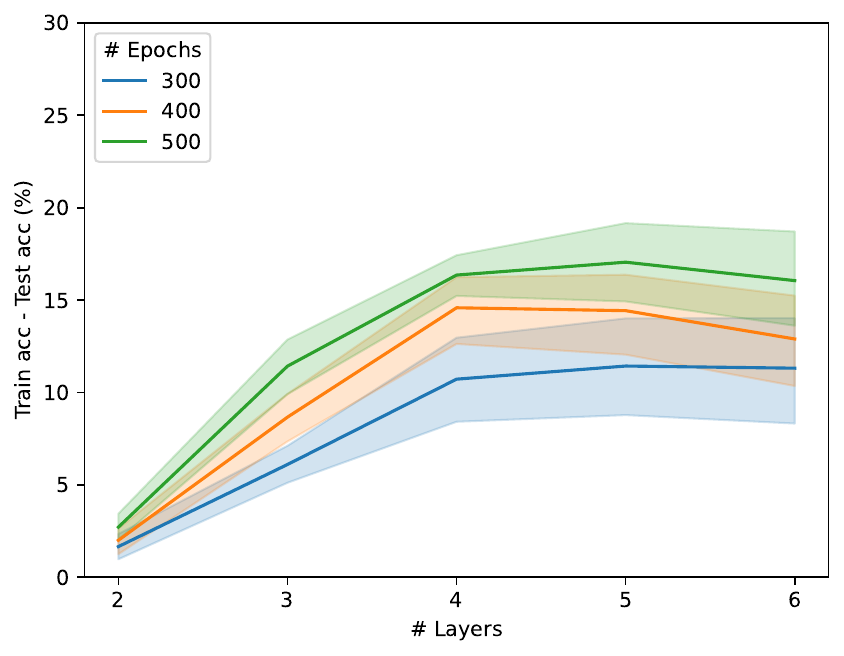}
        \caption{}
        \label{fig:atan_NCI1_d}
    \end{subfigure}\caption{Results on the task \textbf{E1}  for GNNs with activation function $\mathsf{arctan}$. }
    \label{fig:atan_NCI1}
\end{figure}

\begin{figure}[h!]
    \begin{subfigure}{0.49\linewidth}
        \includegraphics[width=\textwidth]{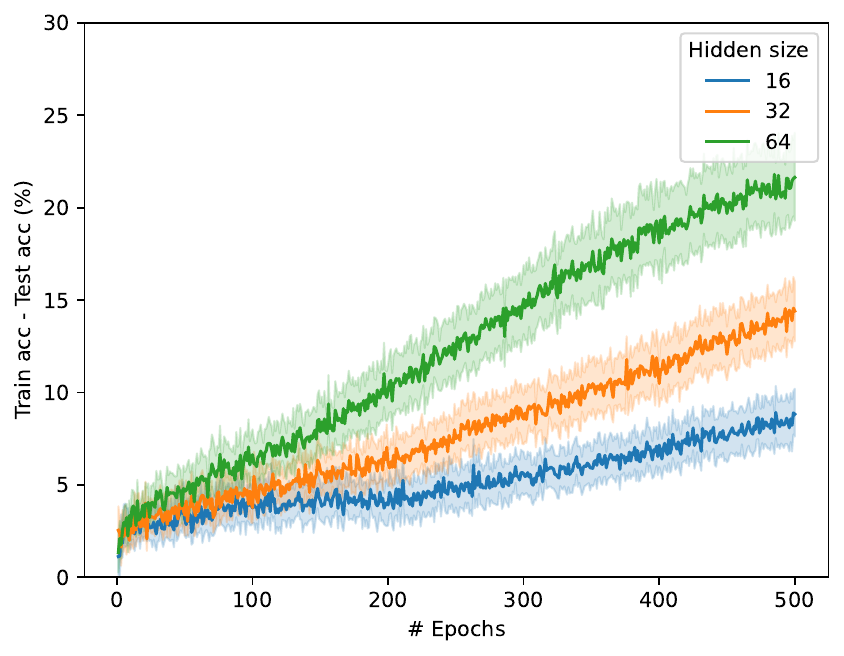}
        \caption{}
        \label{fig:tanh_NCI1_a}
    \end{subfigure}
    \begin{subfigure}{0.49\linewidth}
        \includegraphics[width=\textwidth]{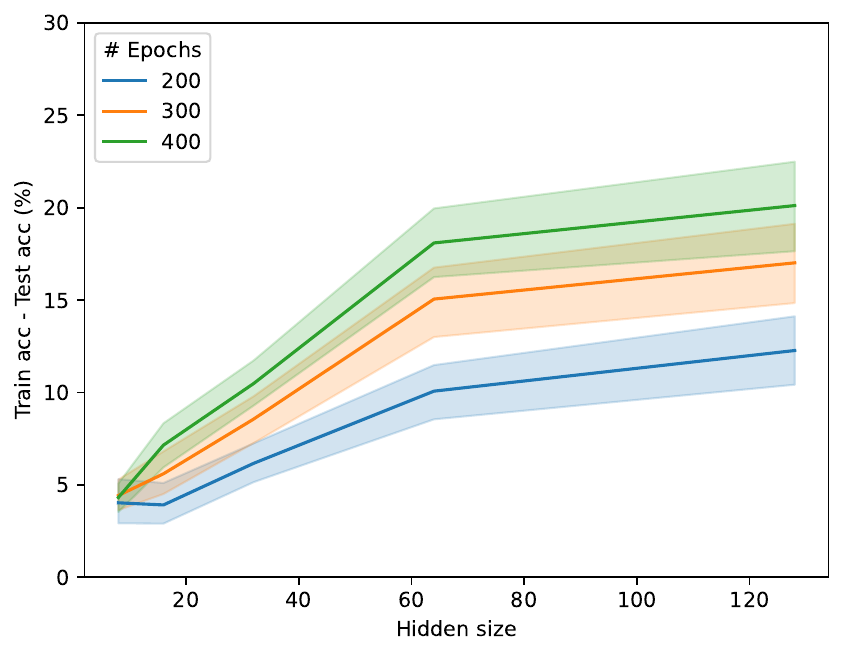}
        \caption{}
        \label{fig:tanh_NCI1_b}
    \end{subfigure}
    \begin{subfigure}{0.49\linewidth}
        \includegraphics[width=\textwidth]{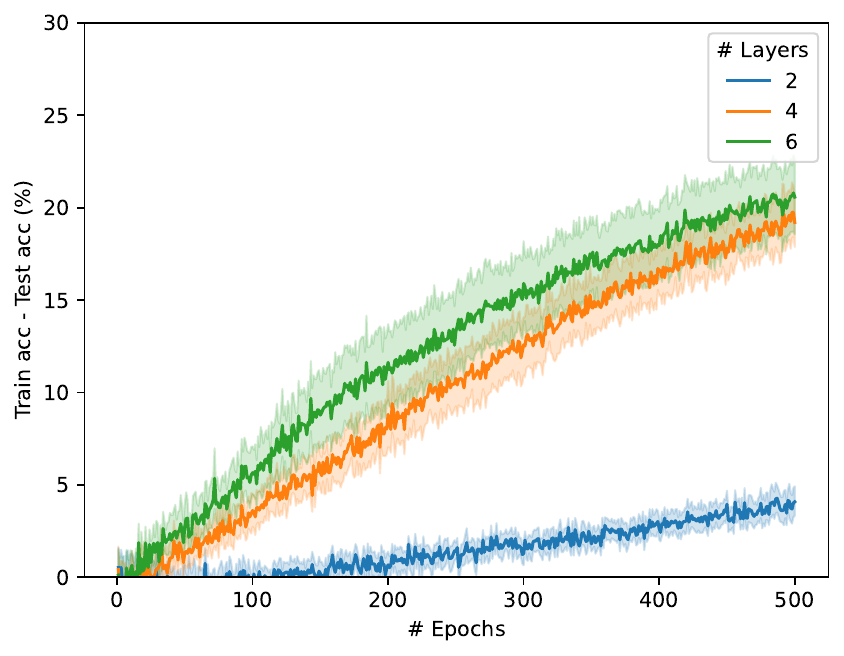}
        \caption{}
        \label{fig:tanh_NCI1_c}
    \end{subfigure}
    \begin{subfigure}{0.49\linewidth}
        \includegraphics[width=\textwidth]{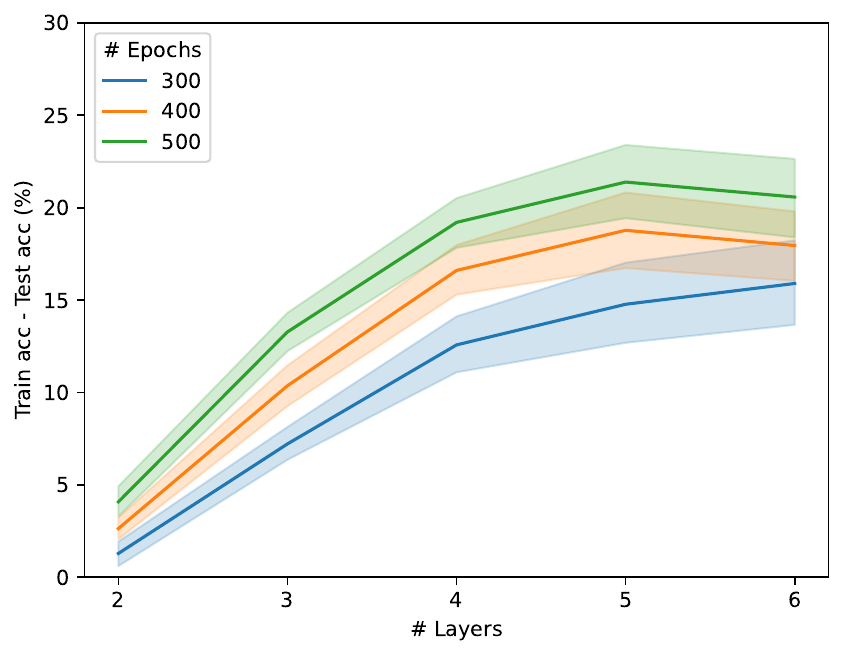}\caption{}
        \label{fig:tanh_NCI1_d}
    \end{subfigure}\caption{Results on the task \textbf{E1}  for GNNs with activation function $\mathsf{tanh}$.}
\label{fig:tanh_NCI1}
\end{figure}

The behaviour of the evolution of $\mathsf{diff}$ proves to be consistent with the bounds provided by Theorem \ref{th:order_Pfaff_general} with respect to increasing the hidden dimension or the number of layers. Although it is hard to establish a precise function that links the VC dimension to $\mathsf{diff}$, given also the complex nature of Pfaffian functions, we can partially rely on Eq.~\eqref{eq:diff_bound} (which is valid for large sample sets) to argue that our bounds are verified by this experimental setting.

\paragraph{Task E2 ---}Numerical results for this task on the \textbf{NCI1} dataset are reported in Figure \ref{fig:exp_colors}, considering the $\mathsf{tanh}$ activation function. In particular, the evolution of $\mathsf{diff}$ is shown in (a) as the number of epochs varies, for different values of $\frac{V(G)}{C^t(G)}$, keeping fixed $l=4$ and $\mathsf{hd}=16$; instead, the evolution of $\mathsf{diff}$ as the ratio $\frac{V(G)}{C^t(G)}$ increases is depicted in (b), for the number of epochs varying in $\{1000,1500,2000\}$.

\begin{figure}[h]
        \begin{subfigure}{0.49\linewidth}
        \includegraphics[width=\textwidth]{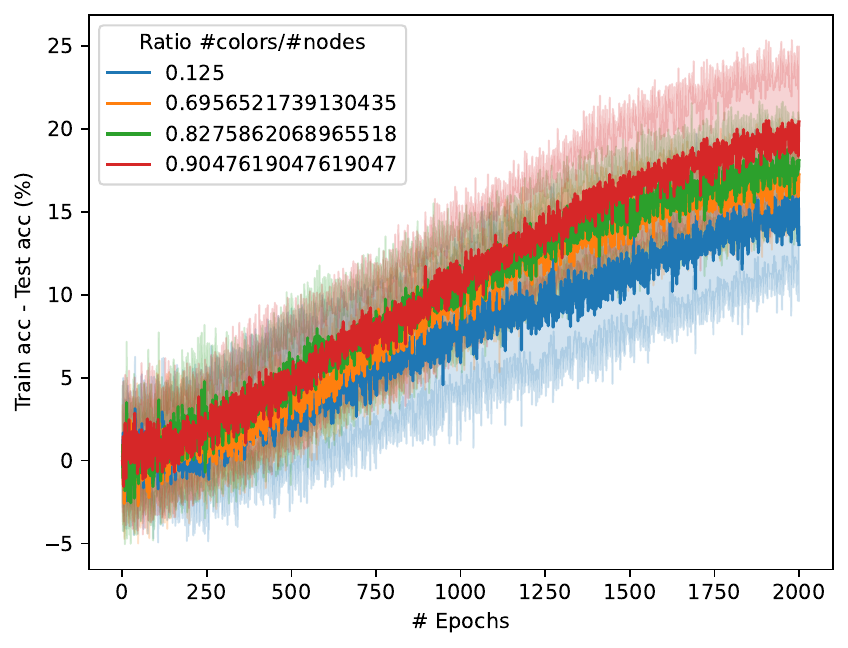}
        \caption{}
        \label{fig:colors_epochs}
    \end{subfigure}
    \begin{subfigure}{0.49\linewidth}
        \includegraphics[width=\textwidth]{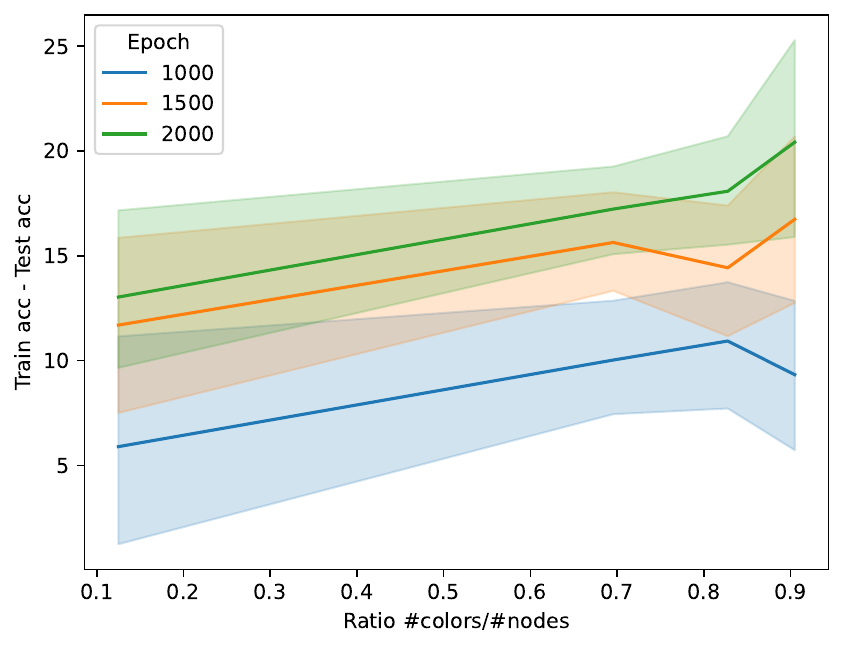}
        \caption{}
        \label{fig:colors_ratio}
    \end{subfigure}
\caption{Results on the task \textbf{E2}  for GNNs with activation function $\mathsf{tanh}$. }\label{fig:exp_colors}
\end{figure}

Similar observations as for the experimental setting \textbf{E1} can be drawn here: indeed, the evolution of $\mathsf{diff}$ in our experiment is consistent with the bounds presented in Theorem \ref{th:vcdim_colors}, as the ratio between colors and nodes increases.

\section{Discussion}\label{sec:discussion}

In this work we derived new bounds for the VC dimension of modern message passing GNNs with Pfaffian activation functions, closing the gap present in the literature with respect to many common used activation functions; furthermore, we propose a preliminary experimental validation to demonstrate the coherence between theory and practice. 

Different research perspectives can be envisaged to improve the results obtained: first, our analysis lacks the derivation of \textit{lower bounds}, which could provide a more precise intuition of the degradation of generalization capabilities for GNNs within the chosen  architectural framework. In addition, providing a relationship between the VC dimension and the difference between the training and test accuracy would be much more informative; we could establish a quantitative measure with respect to the number of parameters that would allow us to better explain the experimental performance. Finally, the proposed analysis on the VC dimension deserves to be extended to other GNN paradigms, such as Graph Transformers \cite{yun2019graph} and Graph Diffusion Models \cite{zhang2023survey}.

\bibliographystyle{unsrt}
\bibliography{references}

\appendix
\section{Proof of the main results}\label{Appendix:proofs}
 The proof of Theorems \ref{th:main_general},\ref{th:main} and \ref{th:vcdim_colors} adopts the same reasoning used
in \cite{karpinski1997polynomial} and  \cite{scarselli2018vapnik} to derive a bound on the VC dimension of  feedforward neural networks and the original GNN model, respectively.
Before proceeding with the proofs,  let us  introduce the required
notation and some results from  \cite{karpinski1997polynomial}. These results  will provide us 
with the mathematical tools to represent the computation of  GNNs with a set 
of Pfaffian equations and to bound the VC dimension based on the format of the functions involved in such equations. 

\subsection{Notation and results from the literature} \label{sec:pfaffianLiterature}

\subsubsection*{Representing a set of equations by a logical formula}
Formally, we use a theory in which a classifier is described  by a logical formula
that  is constructed by combining Pfaffian equations. Thus, let $\tau_1, \dots, \tau_{\Bar{s}}$ be a set of $C^{\infty}$ (infinitely differentiable) functions from $\mathbb{R}^{\gamma+p}$ to $\mathbb{R}$. Suppose that $\Phi(\mathbf{y},\btheta), \mathbf{y} \in \mathbb{R}^{\gamma}$, $\btheta \in \mathbb{R}^{p}$ is a quantifier--free logical formula constructed using the operators  \textit{and} and \textit{or}, and atoms in the form of  $\tau_i (\mathbf{y}, \btheta)=0$.  Note that, fixed $\btheta$, $\Phi(\cdot, \btheta)$ takes as input a vector $\mathbf{y}$ and returns a logical value, so that it can be considered as a classifier with input $\mathbf{y}$ and parameters $\btheta$.  Moreover, the formula can be also used to represent a set of Pfaffian equations, which corresponds to the case when $\Phi$ includes only the operator \textit{and} and $\tau_i$ are Pfaffian. Actually, later, we will see that $\tau_1, \dots, \tau_{\Bar{s}}$ can be specified so that $\Phi$ defines the computation of a GNN.

\subsubsection*{The VC dimension of $\Phi$}
The VC dimension of  $\Phi$ can be defined in the usual way. Indeed, $\Phi$ is said to shatter a set $\mathcal{S}= \{ \Bar{\mathbf{y}}_1, \dots , \Bar{\mathbf{y}}_r \}$ if, for any set of binary assignments $\delta = [\delta_1, \dots, \delta_r ] \in \{ 0,1\}^r$, there exist parameters $\Bar{\btheta}$ such that, for any $i$, $\Phi(\mathbf{y}_i,\Bar{\btheta})$ is true if $\delta_i = 1$, and %$\Phi(\mathbf{y}_i,\Bar{\btheta})$ is 
false if $\delta_i=0$. Then, the VC dimension of $\Phi$ is defined as the size of the maximum set that $\Phi$ can shatter, i.e.,
\begin{equation*}
    \mathsf{VCdim}(\Phi) =\max \limits _{\text{$\mathcal{S}$ is shattered by $\Phi$}} |\mathcal{S}|.
\end{equation*}

Interestingly,  the VC dimension of  $\Phi$ 
can be bounded by studying the topological properties of the inverse image, in the parameter domain, of the functions $\tau_i$ \cite{karpinski1997polynomial}. More precisely,  let $\Bar{\mathbf{y}}_1, \dots , \Bar{\mathbf{y}}_z$ be vectors in $\mathbb{R}^{\Bar{\gamma}}$, and $\mathbf{T}: \mathbb{R}^{\Bar{p}} \rightarrow \mathbb{R}^{\Bar{u}}$, $\Bar{u}\leq \Bar{p}$, be defined as 
\begin{equation}
    \mathbf{T}(\Bar{\btheta}) = [\Bar{\tau}_1(\Bar{\btheta}), \dots, \Bar{\tau}_{\Bar{u}}(\Bar{\btheta})],
    \label{eq:TPaff}
\end{equation}
where $\Bar{\tau}_1(\Bar{\btheta}), \dots, \Bar{\tau}_{\Bar{u}}(\Bar{\btheta})$ are functions of the form of $\tau_i (\Bar{\mathbf{y}}_j, \Bar{\btheta})$, i.e., for each $r$, $1 \leq r \leq \Bar{u}$, there exist integers $i$ and $j$ such that $\Bar{\tau}_r(\Bar{\btheta}) = \tau_i (\Bar{\mathbf{y}}_j, \Bar{\btheta})$. Let $[\epsilon_1, \dots, \epsilon_{\Bar{u}}]$ be a regular value~\footnote{We recall that $[\epsilon_1, \dots, \epsilon_{\Bar{u}}]$ is a regular value of $\mathbf{T}$ if either $\mathbf{T}^{-1}([\epsilon_1, \dots, \epsilon_{\Bar{u}}])= \emptyset$ or $\mathbf{T}^{-1}([\epsilon_1, \dots, \epsilon_{\Bar{u}}])$ is a $(\Bar{p}-\Bar{u})$--dimensional $C^{\infty}$--submanifold of $\mathbb{R}^{\Bar{p}}$.} of $\mathbf{T}$ and assume that there exists a positive integer $B$ that bounds the number of connected components of $\mathbf{T}^{-1}(\epsilon_1, \dots , \epsilon_{\Bar{u}})$ and does not depend on the chosen $\epsilon_1, \dots , \epsilon_{\Bar{u}}$ and on the selected $\Bar{\mathbf{y}}_j$. Then, the following proposition holds \cite{karpinski1997polynomial}.

\begin{theorem}\label{th:karpinsky97}
The VC dimension of $\Phi$ is bounded as follows:

\begin{equation*}
   \mathsf{VCdim}(\Phi) \leq 2 \log B + \Bar{p}(16 +2 \log \bar{s} )\,.
\end{equation*}
\end{theorem}

Therefore, Theorem \ref{th:karpinsky97} provides a bound on the VC dimension of $\Phi$ that depends on  the number $\Bar{p}$ of parameters,  the total number $\Bar{s}$ of functions $\tau_i$, and the bound $B$ on the number of connected components of  $\mathbf{T}^{-1}$.

\subsubsection*{A bound on the number of connected components}

A bound $B$ on the number of connected components of $\mathbf{T}^{-1}$ can  be obtained based on known results from the literature.  In particular, the following theorem provides a bound in the case of a set of Pfaffian equations.
  
\begin{theorem}[\cite{gabrielov2004complexity}]\label{th:gabrielov}
Consider a system of equations $\Bar{q}_1(\btheta) =0, \dots , \Bar{q}_k(\btheta) =0$, where $\Bar{q}_i$, $1 \leq i \leq k$, are Pfaffian functions in a domain $G \subseteq \mathbb{R}^{\Bar{p}}$, having a common Pfaffian chain of length $\Bar{\ell}$ and maximum degrees $(\Bar{\alpha}, \Bar{\beta})$. Then the number of connected components of the set $\{ \btheta | \Bar{q}_1(\btheta) =0, \dots, \Bar{q}_k(\btheta) =0 \}$ is bounded by
\begin{equation*}
    2^{\frac{\Bar{\ell}(\Bar{\ell}-1)}{2}+1} (\Bar{\alpha} + 2 \Bar{\beta} -1)^{\Bar{p}-1}((2 \Bar{p}-1)(\Bar{\alpha} + \Bar{\beta})- 2\Bar{p}+2 )^{\Bar{\ell}}.
\end{equation*}
\end{theorem}

\subsection{Proof of Theorem \ref{th:main_general}}\label{proof:main_general}
First, we prove  Theorem \ref{th:main_general}. The proofs of Theorems \ref{th:main} and \ref{th:vcdim_colors} will adopt the same argumentative scheme. As already
mentioned, we will follow the reasoning 
in \cite{karpinski1997polynomial, scarselli2018vapnik}, which consists of three steps.
First, it is shown that
the computation of  GNNs  can be represented by a set of  equations defined by Pfaffian functions. Then, using the format of such Pfaffian functions, Theorem~\ref{th:gabrielov} allows to derive a bound on the number of connected components of the space defined by the equations.
Finally, Theorem~\ref{th:karpinsky97} provides a bound on the VC dimension of GNNs.

\vspace{10pt}

\noindent
Let us define a set of equations $\tau_i(\mathbf{y}, \btheta)=0$
that specifies the computation of the generic GNN model of Eq.~(\ref{def:gnn_upd_g}).  
Here, $\btheta$ collects the GNN parameters,
while $\mathbf{y}$ contains all the variables necessary to define the GNN calculation,
that is, some variables involved in the representation of the GNN input,
i.e., the input graph, and other variables used for the representation of the
internal features of the GNN.

More precisely, let us assume that $\mathsf{COMBINE}^{(1)}$ has $p_{\mathsf{comb}^{(1)}}$ parameters, $\mathsf{COMBINE}^{(t)}$ has $p_{\mathsf{comb}}$ parameters for $2\leq t \leq L$ ,  $\mathsf{AGGREGATE}^{(1)}$ has $p_{\mathsf{agg}^{(1)}}$ parameters, $\mathsf{AGGREGATE}^{(t)}$ has $p_{\mathsf{agg}}$ parameters for $2\leq t \leq L$, and $\mathsf{READOUT}$ has $p_{\mathsf{read}}$ parameters.

Then, the dimension of $\btheta$ is $ p_{\mathsf{comb}^{(1)}} + p_{\mathsf{agg}^{(1)}} + (L-1)(p_{\mathsf{comb}} + p_{\mathsf{agg}}) + p_{\mathsf{read}} $.
Moreover, for a given graph $\mathbf{G}=(G,\mathbf{L})$  in $\mathcal{G}$,
 $\mathbf{y}$ contains some vectorial representation
of $\mathbf{G}$, namely the $Nq$  graph attributes in $\mathbf{L}_G$, and a vectorial representation of the adjacency matrix $\mathbf{A}$, which requires $N(N-1)/2$ elements.  
Besides, to define the equations, we use the same trick as in \cite{karpinski1997polynomial} and 
introduce new variables for each computation unit of the  network. These variables belong to
the input $\mathbf{y}$ of $\tau$.
Formally, we consider a vector of $d$ variables $\mathbf{h}_{v}^{(k)}$ for each node $v$ and for each layer $k$. Note that, as we may be interested in defining
multiple GNN computations on multiple graphs at the same time, here $v$ implicitly addresses a specific node of some graph in the domain. Finally, a variable $\mathsf{READOUT}$ for each graph contains
just a single output of the GNN. Thus, in total, the dimension of $\mathbf{y}$ is 
$Nq+N(N-1)/2+ NdL+1$. 

Therefore, the computation of the GNN model in~(\ref{def:gnn_upd_g}) is  defined by the following set of $LNd+Nq+1$ equations,

\begin{equation}\label{eq:tau1gen}
	\mathbf{h}_{v}^{(0)}-\mathbf{L}_{v}= 0, \end{equation}
 \begin{equation}
	\label{eq:tau2gen}
	\mathbf{h}_{v}^{(t+1)}- \mathsf{COMBINE}^{(t+1)}\big(  \mathbf{h}_v^{(t)}, \mathsf{AGGREGATE}^{(t+1)} (\lms \mathbf{h}_u^{(t)} | u \in \mathsf{ne}(v) \rms, \mathbf{A})   \big)= 0,
 \end{equation}
 \begin{equation}
 \label{eq:tau3gen}
\mathsf{\overline{READOUT}}- \mathsf{READOUT}(\lms \mathbf{h}_v^{(L)}:v\in V \rms ) = 0,
\end{equation}
where $\mathbf{A}$ is the variable storing the adjacency matrix of the input graph. We can assume that $\mathbf{A}$ is valid for any finite graph.

The following lemma specifies the format of the Pfaffian functions involved in Eqs.~(\ref{eq:tau1gen})--(\ref{eq:tau3gen}).
\begin{lemma}\label{general:main_lemma}
Let $\mathsf{COMBINE}^{(t)}$, $\mathsf{AGGREGATE}^{(t)}$ and $\mathsf{READOUT}$ be Pfaffian functions with format, respectively, $(\alpha_{\mathsf{comb}}, \beta_{\mathsf{comb}}, \ell_{\mathsf{comb}})$, $(\alpha_{\mathsf{agg}}, \beta_{\mathsf{agg}}, \, \ell_{\mathsf{agg}})$, $(\alpha_{\mathsf{read}}, \beta_{\mathsf{read}}, \ell_{\mathsf{read}})$
w.r.t. the variables $\mathbf{y}$ and $\btheta$ described above, then:
\begin{enumerate}
\item the left part of Eq.~(\ref{eq:tau1gen}) is a polynomial of degree $1$;
\item the left part of Eq.~(\ref{eq:tau2gen}) is a Pfaffian function with format  \\
$(\alpha_{\mathsf{agg}} + \beta_{\mathsf{agg}} -1 + \alpha_{\mathsf{comb}} \beta_{\mathsf{agg}} , \beta_{\mathsf{comb}}, \ell_{\mathsf{comb}} + \ell_{\mathsf{agg}} )$;
\item the left part of Eq.~(\ref{eq:tau3gen}) is a Pfaffian function with format \\
$(\alpha_{\mathsf{read}}, \beta_{\mathsf{read}}, \ell_{\mathsf{read}})$;

\item Eqs.~(\ref{eq:tau1gen})--(\ref{eq:tau3gen}) constitute a system of Pfaffian equations with a maximal  format 
$(\alpha_{\mathsf{system}},\beta_{\mathsf{system}},\ell _{\mathsf{system}})$, 
where $\alpha_{\mathsf{system}} = \max \{ \alpha_{\mathsf{agg}} + \beta_{\mathsf{agg}} -1+\alpha_{\mathsf{comb}} \beta_{\mathsf{agg}},  \alpha_{\mathsf{read}}\}$, $\beta_{\mathsf{system}} = \max \{ \beta_{\mathsf{comb}}, \beta_{\mathsf{read}}\}$ and
$\ell_{\mathsf{system}}=LNd(\ell_{\mathsf{comb}} + \ell_{\mathsf{agg}} )+ \ell_{\mathsf{read}}$.
\end{enumerate}

\end{lemma}

\begin{proof} The first point is straightforwardly evident, while the third is true by definition. The second point can be derived by applying 
the composition lemma for Pfaffian functions \cite{gabrielov1995complexity}, 
 according to which, if two functions
$f$ and $g$ have format $(\alpha_f,\beta_f,\ell_f)$ and $(\alpha_g,\beta_g,\ell_g)$, respectively,
then their composition $f \circ g$ has format $(\alpha_g+\beta_g-1+\alpha_f \beta_g,\beta_f ,\ell_f+\ell_g)$. Finally, the fourth point is obtained by taking the maximum of the format of the involved Pfaffian equations also observing that the common chain is the  concatenation of the chains.
\end{proof}

Now we can proceed with the proof of Theorem \ref{th:main_general}.
\begin{proof}
Let $\mathbf{T}$ be defined as in Eq.~(\ref{eq:TPaff}), where
$\tau_i (\mathbf{y}, \btheta)=0$ are the equations in (\ref{eq:tau1gen}), (\ref{eq:tau2gen}), (\ref{eq:tau3gen}). Combining Theorem~\ref{th:gabrielov} with the formats
provided by point 3. of Lemma~\ref{general:main_lemma}, for any input graph and any value of the variables
$\mathbf{y}$, the number of connected components of $\mathbf{T}^{-1}$  satisfies
\begin{equation}\label{eq:conn_comp_gnn}
  B\leq  2^{\frac{\bar{\ell}(\bar{\ell}-1)}{2}+1} (\Bar{\alpha} + 2 \Bar{\beta} -1)^{\Bar{p}-1}((2 \Bar{p}-1)(\Bar{\alpha} + \Bar{\beta})- 2\Bar{p}+2 )^{\Bar{\ell}},
\end{equation}
where 
$\Bar{p} = p_{\mathsf{comb}^{(1)}} + p_{\mathsf{agg}^{(1)}} + (L-1)(p_{\mathsf{comb}} + p_{\mathsf{agg}}) + p_{\mathsf{read}} $, $\bar{\alpha} = \alpha_{\mathsf{system}}$, $\bar{\beta} = \beta_{\mathsf{system}}$, $\bar{\ell}=\Bar{p}(LNd(\ell_{\mathsf{comb}} + \ell_{\mathsf{agg}} )+ \ell_{\mathsf{read}})$. 

By Theorem \ref{th:karpinsky97}, the VC dimension of the GNN described by Eqs. (\ref{eq:tau1gen})--(\ref{eq:tau3gen}) is bounded by 
\begin{equation}\label{eq:VCdim_B}
  \mathsf{VCdim}(\mathsf{GNN}) \leq 2 \log B + \Bar{p}(16 +2 \log \bar{s} )\,,
\end{equation}
where 
$\bar{s}=LNd+Nq+1$. Thus, substituting Eq. \eqref{eq:conn_comp_gnn} in Eq. \eqref{eq:VCdim_B}, we have:
{\small \begin{align*}
	\mathsf{VCdim}\bigl(\mathsf{GNN}\bigr) &\leq  2\log B + \Bar{p}(16 +2 \log \bar{s} ) \\
  &\leq   2\log \left(  2^{\frac{\bar{\ell}(\bar{\ell}-1)}{2}+1}(\Bar{\alpha} + 2 \Bar{\beta} -1)^{\Bar{p}-1} ((2 \Bar{p}-1)(\Bar{\alpha} + \Bar{\beta})- 2\Bar{p}+2 )^{\Bar{\ell}}\right) + \\
   &\,\,\,\,\; +\Bar{p}(16 +2 \log \bar{s} )\\
  &= \bar{\ell}(\bar{\ell}\!-\!1)\!+\!  
  2({\Bar{p}\!-\!1} )\log \left(\Bar{\alpha} \!+\!2 \Bar{\beta}\!-\!1\right)\!+\!2 \Bar{\ell}\log\left((2 \Bar{p}\!-\!1)(\Bar{\alpha}\!+\!\Bar{\beta})\!-\!2\Bar{p}\!+\!2 \right) + \\
  &\,\,\,\,\;+\Bar{p}(16 +2 \log \bar{s} )+2,
  \end{align*}}
obtaining Eq. \eqref{VCGNN}. \\
If we denote $H = LNd((\ell_{\mathsf{comb}} + \ell_{\mathsf{agg}} )+ \ell_{\mathsf{read}})$, we have:
\begin{eqnarray}\label{eq:final_bound}
    \mathsf{VCdim}\bigl(\mathsf{GNN}\bigr) &\leq  &  \Bar{p }H(\Bar{p}H -1)+ 2({\Bar{p}-1} )\log (\alpha_{\mathsf{system}} + 2\beta_{\mathsf{system}}-1) \nonumber \\
    && + 2\bar{p}H\log \bigl((2\bar{p}-1)(\alpha_{\mathsf{system}} + \beta_{\mathsf{system}})- 2\bar{p} +2\bigr) \nonumber \\
    &&+\Bar{p}(16 +2 \log(\Bar{s}) ) + 2\\
    & \leq & \Bar{p}^2H^2+ 2{\Bar{p}} \log \left(3\gamma \right) \nonumber \\
    && +2\bar{p}H\log \left((4\gamma -2)\Bar{p}+2-2\gamma\right) \nonumber \\
    && +\Bar{p}(16 +2 \log(\Bar{s}) ) + 2. \nonumber
\end{eqnarray}

  Then, by replacing $\Bar{p}$, $H$ and $\Bar{s}$, and setting $\gamma = \max \{ \bar{\alpha}, \bar{\beta} \}$, it follows that:
 {\footnotesize \begin{align*}
	\mathsf{VCdim}\bigl(\mathsf{GNN}\bigr) &\leq   ( p_{\mathsf{comb}^{(1)}}\!+ \!p_{\mathsf{agg}^{(1)}}\!+\!(L\!-\!1)(p_{\mathsf{comb}}\!+\!p_{\mathsf{agg}})\!+\!p_{\mathsf{read}})^2(LNd(\ell_{\mathsf{comb}}\!+\!\ell_{\mathsf{agg}} )\!+ \!\ell_{\mathsf{read}})^2 \\
   &+ 2( p_{\mathsf{comb}^{(1)}} + p_{\mathsf{agg}^{(1)}} + (L-1)(p_{\mathsf{comb}} + p_{\mathsf{agg}}) + p_{\mathsf{read}} ) \log \left(3\gamma \right) \\
  &+ 2( p_{\mathsf{comb}^{(1)}} + p_{\mathsf{agg}^{(1)}} + (L-1)(p_{\mathsf{comb}} + p_{\mathsf{agg}}) + p_{\mathsf{read}} )\cdot \\
  & \cdot \log\left((4\gamma -2)(p_{\mathsf{comb}^{(1)}} + p_{\mathsf{agg}^{(1)}} + (L-1)(p_{\mathsf{comb}} + p_{\mathsf{agg}}) + p_{\mathsf{read}} )+2-2\gamma\right) \\
  &+( p_{\mathsf{comb}^{(1)}}\!+\! p_{\mathsf{agg}^{(1)}}\!+ (L\!-\!1)(p_{\mathsf{comb}}\!+\!p_{\mathsf{agg}})\!+ \!p_{\mathsf{read}} )(16\!+\!2 \log(LNd\!+\!Nq\!+\!1 ))\,,
  \end{align*}}
  which leads to Eq. \eqref{VCGNN_ext} as in the thesis.
\end{proof}

\subsection{Proof of Theorem \ref{th:order_Pfaff_general}}
% \ad{Come Franco ha fatto giustamente notare, andava fatta una sezione a parte per il Teorema 2}

The orders of growth of the VC dimension w.r.t. $\bar{p},N,L,d,q$ are straighforwardly derived by inspecting Eq.~\ref{VCGNN_ext} in Theorem \ref{th:main_general}.

\subsection{Proof of Theorem \ref{th:main}}

As already mentioned in \ref{proof:main_general}, the proof of Theorem \ref{th:main} follows the same scheme used in proof of Theorem \ref{th:main_general}. We just recall that we consider a GNN defined by the following updating equation:
\begin{equation}\label{gnnsimple}
    \mathbf{h}_v^{(t+1)} = \bsigma ( \mathbf{W}_{\mathsf{comb}}^{(t+1)}\mathbf{h}_v^{(t)} + \mathbf{W}_{\mathsf{agg}}^{(t+1)}\mathbf{h}_{\mathsf{ne}(v)}^{(t+1)} +\mb^{(t+1)}),
\end{equation}
where $\bsigma$ is the activation function and 
\begin{equation}
    \mathbf{h}_{\mathsf{ne}(v)}^{(t+1)} = \sum\limits_{u\in \text{ne}(v)}\mathbf{h}_u^{(t)}.
\end{equation}
The hidden states are initialised as $\mathbf{h}_{v}^{(0)}=\mathbf{L}_{v}$. 
It is easily seen that the total number of parameters is  $p=(2d+1)(d(L-1)+q+1) - q$. Indeed, the parameters are:
\begin{itemize}
    \item $\mathbf{W}_{\mathsf{comb}}^{(1)}, \mathbf{W}_{\mathsf{agg}}^{(1)} \in \mathbb{R}^{d \times q}, \mb^{(1)} \in \mathbb{R}^{d \times 1 }$, so that we have $2dq+d$ parameters;
    \item $\mathbf{W}_{\mathsf{comb}}^{(t)}, \mathbf{W}_{\mathsf{agg}}^{(t)} \in \mathbb{R}^{d \times d}, \mb^{(t)} \in \mathbb{R}^{d\times 1}$ for $t=2, \dots, L$,  so that we have $(2d^2+d)(L-1)$ parameters;
    \item $\mathbf{w} \in \mathbb{R}^{1 \times d}, b \in \mathbb{R}$, so that we have $d+1$  parameters.
\end{itemize}
Summing up, we will have $2dq+d + (2d^2+d)(L-1) + d+1 = (2d+1)(d(L-1)+q+1) - q$ parameters.
By Eq. (\ref{gnnsimple}),  the computation of the GNN is straightforwardly  defined by the following set of $LNd+Nq+1$ equations: 
\begin{eqnarray}\label{eq:tau1}
	\mathbf{h}_{v}^{(0)}-\mathbf{L}_{v}&=&0,\\
	\label{eq:tau2}
	\mathbf{h}_{v}^{(t+1)}-\bsigma\Big(\mathbf{W}_{\mathsf{comb}}^{(t+1)}\mathbf{h}_{v}^{(t)} +
	\sum_{u}\mathbf{W}_{\mathsf{agg}}^{(t+1)} \mathbf{h}_{u}^{(t)} m_{v,u} + \mb^{(+1t)}
	\Big)&=&0,\\
\mathsf{READOUT}-\bsigma\Big( \sum_{v\in V}\mathbf{w} \mathbf{h}_{v}^{(L)}+b\ \Big)&=&0,
	\label{eq:tau3}
\end{eqnarray}
where $m_{v,u}$ is a binary value, which is $1$ when $v$ and $u$ are connected and $0$, otherwise.

Now we state the analogous of Lemma \ref{general:main_lemma} to retrieve the format of Pfaffian functions involved in the above equations. 

\begin{lemma}\label{simple:main_lemma}
Let $\bsigma$ be a Pfaffian function in $\mathbf{x}$ with format $(\alpha_\sigma,\beta_\sigma,\ell_\sigma)$, then
w.r.t. the variables $\mathbf{y}$ and $\mathbf{w}$ described above,
\begin{enumerate}
\item the left part of Eq.~(\ref{eq:tau1}) is a polynomial of degree $1$;
\item the left part of Eq.~(\ref{eq:tau2}) is a Pfaffian function having format $(2+3\alpha_\sigma,\beta_\sigma,\ell_\sigma)$;
\item the left part of Eq.~(\ref{eq:tau3}) is a Pfaffian function having format $(1+2\alpha_\sigma,\beta_\sigma,\ell_\sigma)$;
\item  Eqs.~(\ref{eq:tau1})--(\ref{eq:tau3}) constitute a system of Pfaffian equations with format 
$(2+3\alpha_\sigma,\beta_\sigma, H \ell_\sigma)$,
where the shared chain is obtained by concatenating the 
chains of $H=LN d+1$ equations in ~(\ref{eq:tau2}),(\ref{eq:tau3}),
including an activation function.
\end{enumerate}
\end{lemma}
\begin{proof} The first point is straightforwardly evident. The second and third points can be derived by applying 
the composition lemma for Pfaffian functions.
Actually, the formula inside $\bsigma$ in Eq.~(\ref{eq:tau2}) is a polynomial of degree 3, 
due to the factors $\mathbf{h}_{u}^{(t-1)}\mathbf{W}_{\mathsf{agg}}^{(t)}m_{v,u}$, while the formula inside $\bsigma$ in Eq.~(\ref{eq:tau3}) is a polynomial of degree 2, due to the factors $\mathbf{h}_{v}^{(L)}\mathbf{W}$.
Moreover, polynomials are Pfaffian functions with null chains, $\alpha$ equals $0$ and $\beta$ equals their degrees. Thus, the functions inside $\bsigma$ in 
Eqs.~(\ref{eq:tau2}) and ~(\ref{eq:tau3}) have format $(0,3,0)$ and $(0,2,0)$, respectively.
Then, the thesis follows by the composition lemma \cite{gabrielov1995complexity}, 
 according to which if two functions
$f$ and $g$ have format $(\alpha_f,\beta_f,\ell_f)$ and $(\alpha_g,\beta_g,\ell_g)$, respectively,
then their composition $f \circ g$ has format $(\alpha_g+\beta_g-1+\alpha_f \beta_g,\beta_f ,\ell_f+\ell_g)$. Finally, the fourth point is a consequence of the fact that the equations are independent and 
the chains can be concatenated. The length of the chain derives directly from the existence of $H=LNd +1$  equations using $\bsigma$. The degree is obtained copying the largest 
degree of a Pfaffian function, which is the one in Eq. (\ref{eq:tau2}).
\end{proof}

As in \ref{proof:main_general}, it is enough to combine Lemma \ref{simple:main_lemma}, Theorem \ref{th:gabrielov} and Theorem \ref{th:karpinsky97} to obtain the bounds stated in the thesis. The bounds on the VC dimension of the specific GNN with $\mathsf{logsig}$ as activation function is easily found since the format of $\mathsf{logsig}$ is (2,1,1):

 {\footnotesize \begin{align*}	\mathsf{VCdim}\bigl(\mathsf{GNN}\bigr) &\leq   ((2d+1)(d(L-1)+q+1) - q)^2(LN d+1)^2 \\
   &+ 2((2d+1)(d(L-1)+q+1) - q )\log \left(9\right) \\
  &+ 2((2d+1)(d(L-1)+q+1) - q)\log\left(16((2d+1)(d(L-1)+q+1) - q)\right) \\
  &+((2d+1)(d(L-1)+q+1) - q)(16 +2 \log(LNd+Nq+1 ))\,.
  \end{align*}}
\qed

Let us call Basic GNN (BGNN) the model
of Eqs.~\eqref{eq:morrisGNN}--\eqref{eq:morrisREADOUT}.  The proof 
is based on the introduction of an extended version,
 which we call EGNN, that  can simulate the BGNN. Due to this capability,
the EGNN can shatter any set that is shattered by the BGNN so that
its  VC dimension is greater or equal to the VC dimension of the BGNN.
The proof will follow by bounding the VC dimension of the former model. \\
More precisely, the EGNN exploits the same aggregation mechanism of the BGNN
to compute the features, which is described by Eq.~\eqref{eq:morrisGNN}.
On the other hand, the  $\mathsf{READOUT}$ function is defined as
\begin{equation}\label{eq:Extended_READOUT}
    \mathsf{READOUT}\Big( \lms \mathbf{h}_v^{(L)} \; | \; v \in V  \rms \Big):=  f \Big( \sum_{v\in V} \mw \mathbf{h}_{v}^{(L)} c_v + b \Big),
\end{equation}
where $c_v$ are additional real inputs used to weight each node feature in the $\mathsf{READOUT}$ function.
The simulation is based on the following steps.
\begin{itemize}
\item[1)] Each input graph $G$ of the BGNN is transformed to another graph $G^\prime$, where all the nodes having the same 1--WL color  are merged into a single node and the edges are merged consequently;
\item[2)] The EGNN is applied to $G^\prime$ and each $c_v$ is set equal to the number of nodes that have been merged to obtain node $v$.
\end{itemize}
Note that a GNN cannot distinguish nodes with the same color as the computation is the same on all these nodes. Thus, the BGNN and the EGNN produce the same features on  nodes sharing color. As a consequence,
also the $\mathsf{READOUTs}$ of the two models have the same output, when the $c_v$ are  equal to the number of nodes within each color cluster. \\
Given these assumptions, the number of equations describing the Pfaffian variety associated to the EGNN is reduced to $s_c=C_1d+C_0q+1$, which can be used in place of $\Bar{s}$ in Theorem \ref{th:karpinsky97}.  Moreover, also the chains of the Pfaffian functions in merged  equations can be merged and we have that $H$ can be replaced by $H_c =C_1 d+1$.
Finally, the length of the chain $\ell$ of Theorem~\ref{th:gabrielov}
is replaced by $\bar{\ell}_c=\Bar{p}H_c\ell_\sigma$.

With the above changes, we can replace the variables in Eq.~(\ref{eq:final_bound}) as in \ref{proof:main_general}, obtaining:
\begin{eqnarray*}
\mathsf{VCdim}\bigl(\mathsf{GNN}\bigr) &\leq  &  \Bar{p }H_c(\Bar{p}H_c -1)+   
  2\Bar{p} \log \left(9\right) \\
  &&+ 2\bar{p}H_c\log\left(16\Bar{p}-7\right) \\
  &&+\Bar{p}(16 +2 \log(\Bar{s}_c) ) +2 \\
 &\leq  &  \Bar{p }^2(C_1d+1)^2+   
  2\Bar{p} \log \left(9\right) \\
  &&+ 2\bar{p}(C_1d+1)\log\left(16\Bar{p}-7\right) \\
  &&+\Bar{p}(16 +2 \log(C_1d+C_0q+1) ).
  \end{eqnarray*}
and the thesis holds. \qed

\section{Experiments on other datasets} 
In this appendix, we report the additional results on the experiment E1, regarding the evolution of the difference between the training and the test set, for GNNs with activation function $f \in \{\mathsf{arctan}, \mathsf{tanh}\}$, over a dataset $\mathcal{D} \in \{  \text{\textbf{PROTEINS}},\text{\textbf{PTC-MR}} \}.$ 
Each figure shows the evolution of $\mathsf{diff}$ through the epochs, for certain values of $\mathsf{hd}$, keeping fixed $l=3$, and for certain values of $l$, keeping fixed $\mathsf{hd}=32$; for each figure, the picture on the left shows how $\mathsf{diff}$ evolves as the hidden size increases, while the picture on the right shows how $\mathsf{diff}$ evolves as the hidden size, or the number of layers, increases.\label{Appendix:experiments}

The experiments reported here confirm the same conclusion drawn in Section \ref{sec:experiments}. 

% \subsection{Results for task \textbf{E1}: GNNs with activation function $\mathsf{arctan}$}

\begin{figure}[ht!]\label{fig:atan_proteins}
    \begin{subfigure}{0.49\linewidth}
        \includegraphics[width=\textwidth]{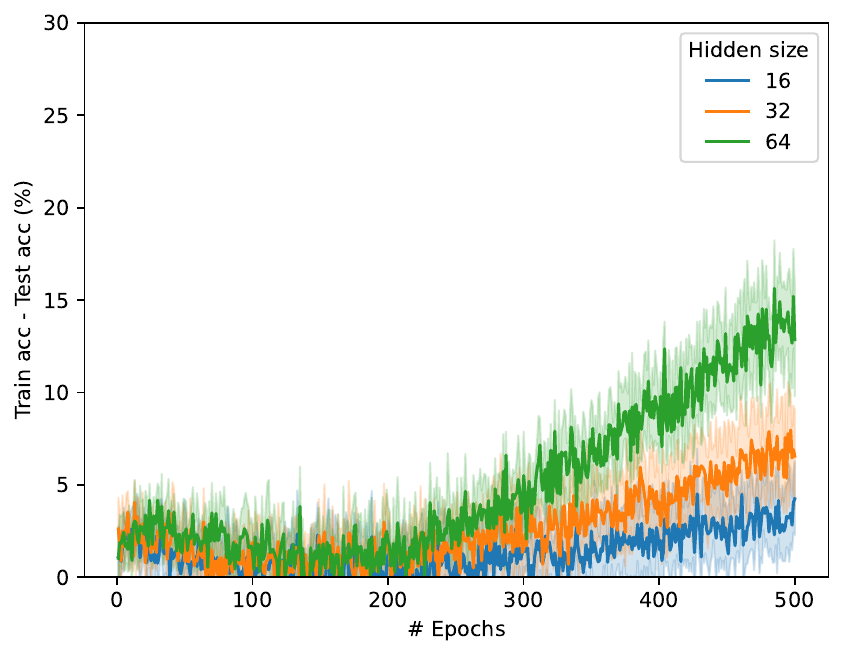}
    \end{subfigure}
    \begin{subfigure}{0.49\linewidth}
        \includegraphics[width=\textwidth]{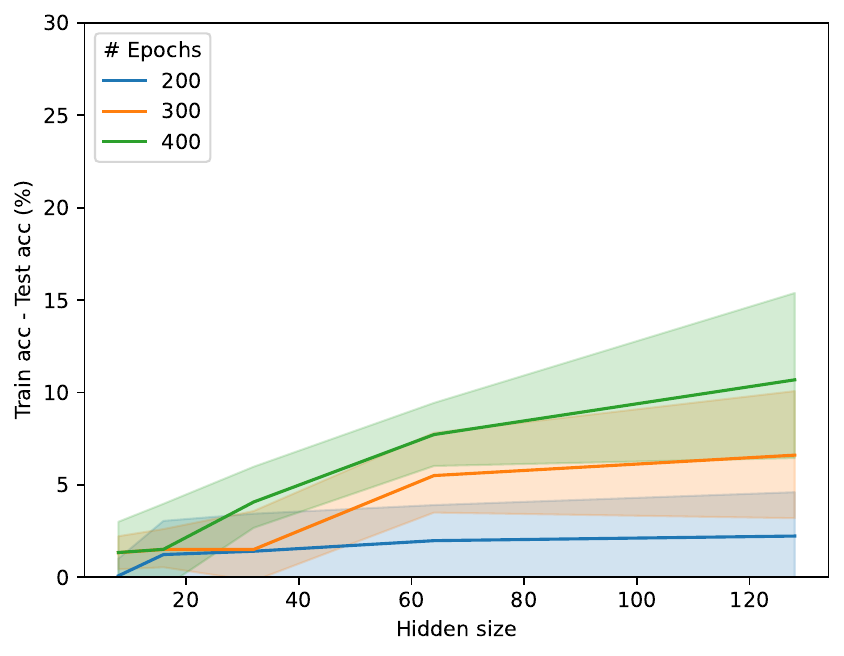}
    \end{subfigure}
    \begin{subfigure}{0.49\linewidth}
        \includegraphics[width=\textwidth]{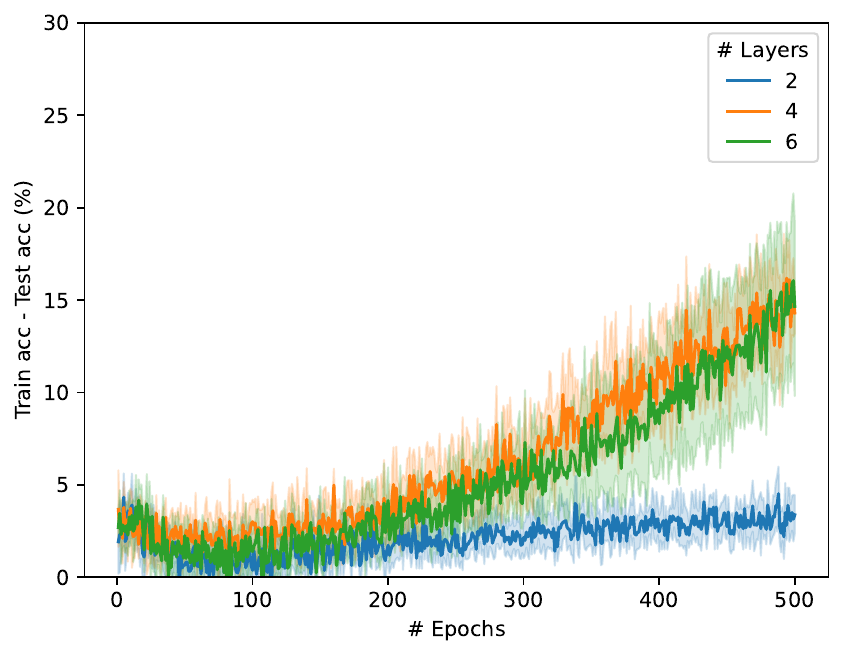}
    \end{subfigure}
    \begin{subfigure}{0.49\linewidth}
        \includegraphics[width=\textwidth]{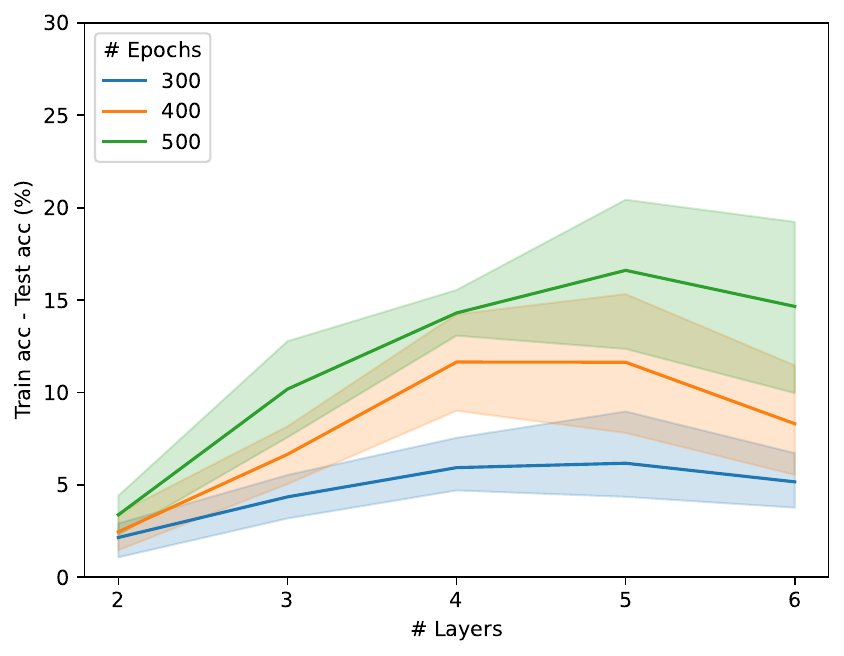}
    \end{subfigure}\caption{Results on the task \textbf{E1} for GNNs with activation function $\mathsf{atan}$ over the dataset \textbf{PROTEINS}.}
\end{figure}

\begin{figure}[ht!]\label{fig:atan_ptc}
    \begin{subfigure}{0.49\linewidth}
        \includegraphics[width=\textwidth]{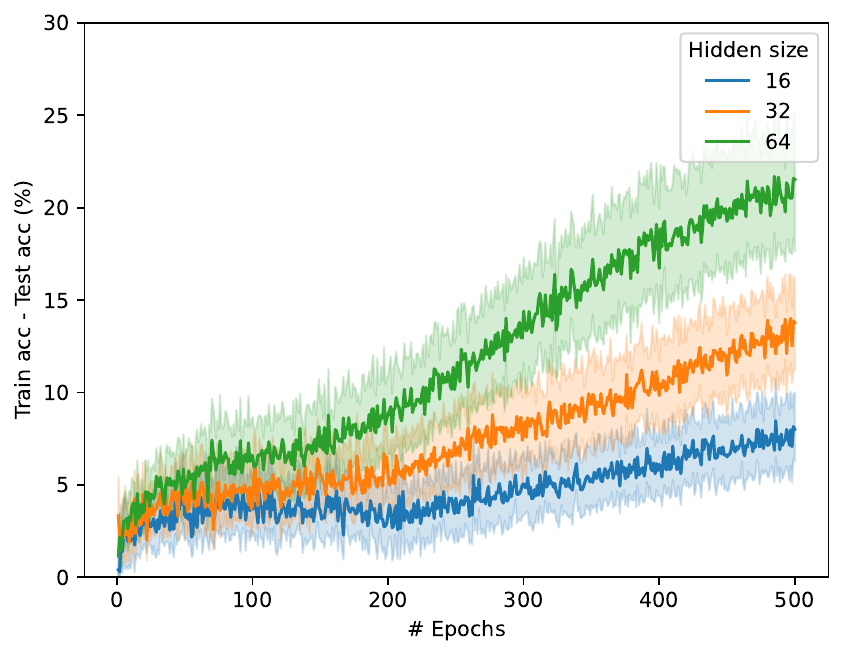}
    \end{subfigure}
    \begin{subfigure}{0.49\linewidth}
        \includegraphics[width=\textwidth]{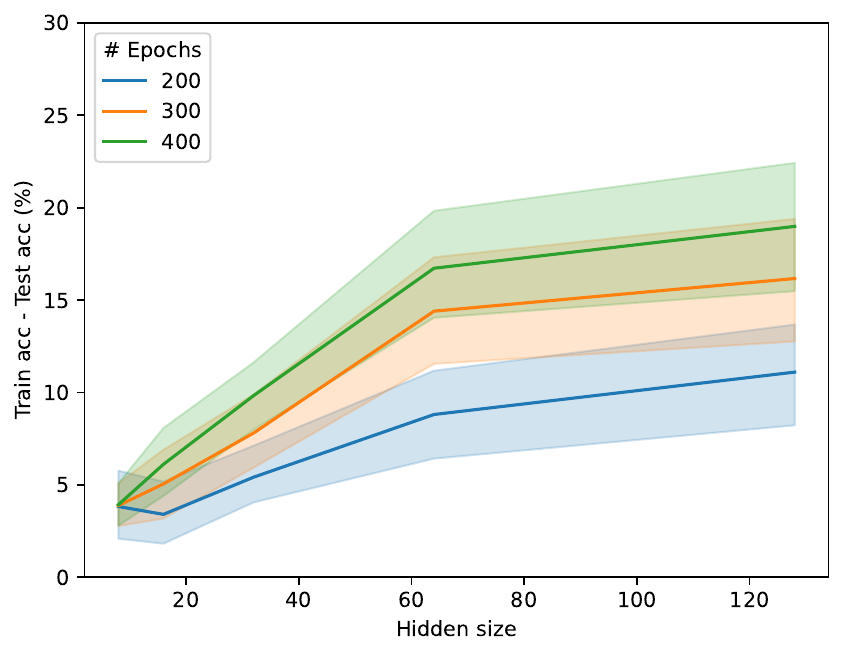}
    \end{subfigure}
    \begin{subfigure}{0.49\linewidth}
        \includegraphics[width=\textwidth]{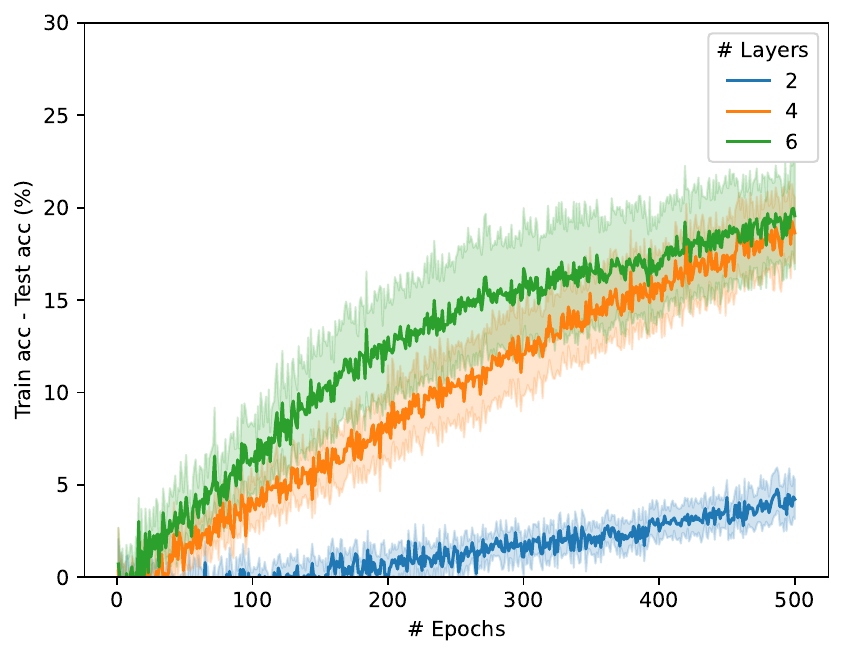}
    \end{subfigure}
    \begin{subfigure}{0.49\linewidth}
        \includegraphics[width=\textwidth]{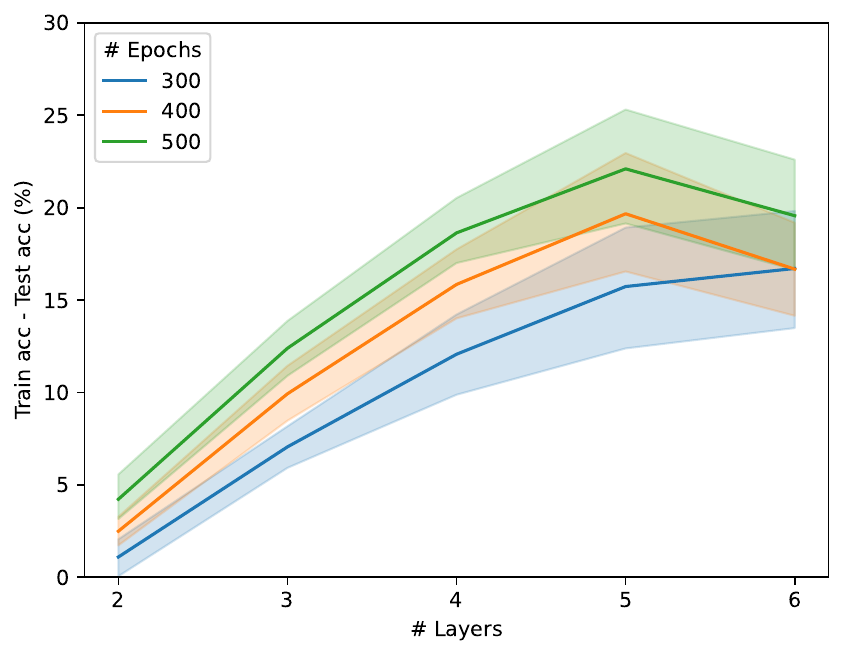}
    \end{subfigure}\caption{Results on the task \textbf{E1} for GNNs with activation function $\mathsf{atan}$ over the dataset \textbf{PTC-MR}.}
\end{figure}

% \subsection{Results for task \textbf{E1}: GNNs with activation function $\mathsf{tanh}$}

\begin{figure}[ht!]\label{fig:tanh_proteins}
    \begin{subfigure}{0.49\linewidth}
        \includegraphics[width=\textwidth]{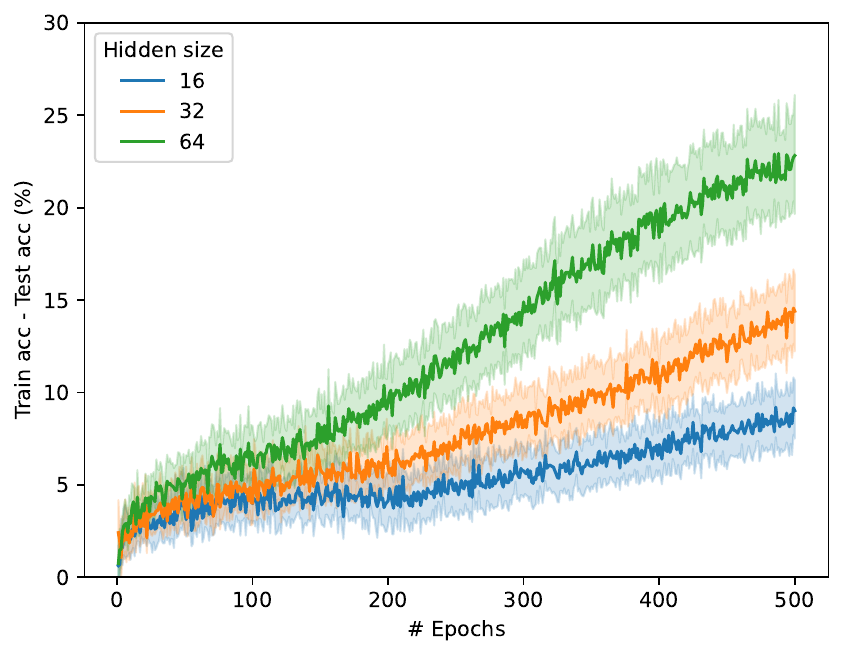}
    \end{subfigure}
    \begin{subfigure}{0.49\linewidth}
        \includegraphics[width=\textwidth]{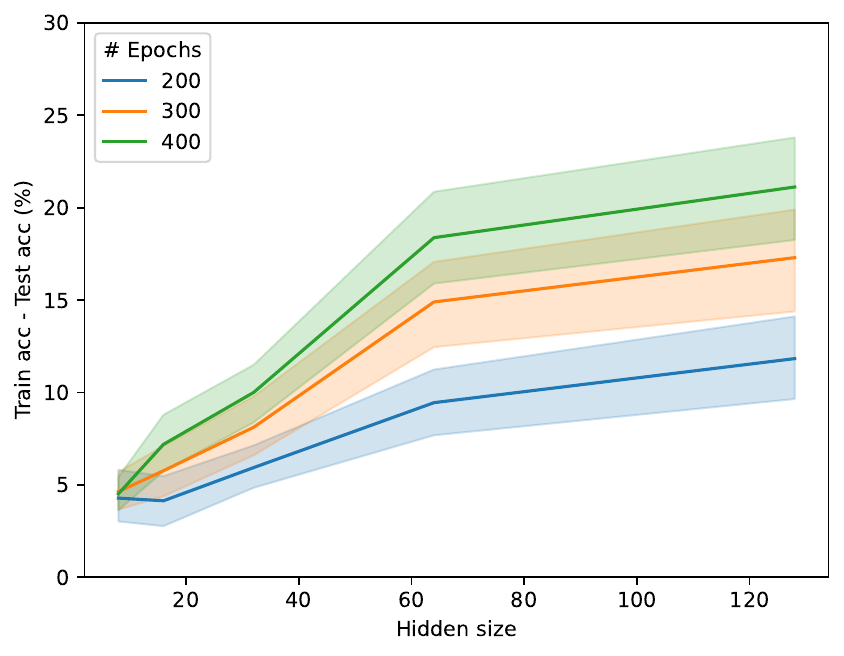}
    \end{subfigure}
    \begin{subfigure}{0.49\linewidth}
        \includegraphics[width=\textwidth]{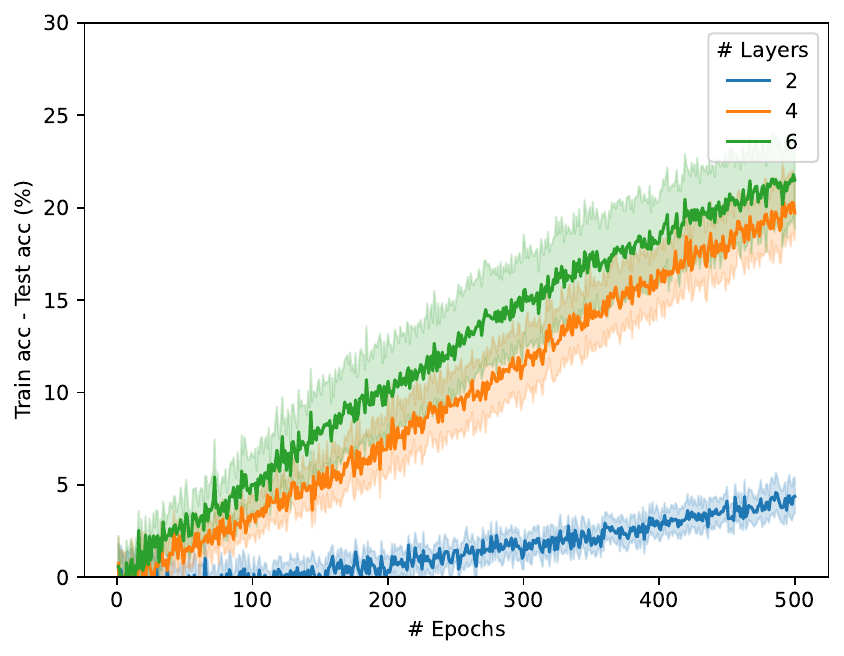}
    \end{subfigure}
    \begin{subfigure}{0.49\linewidth}
        \includegraphics[width=\textwidth]{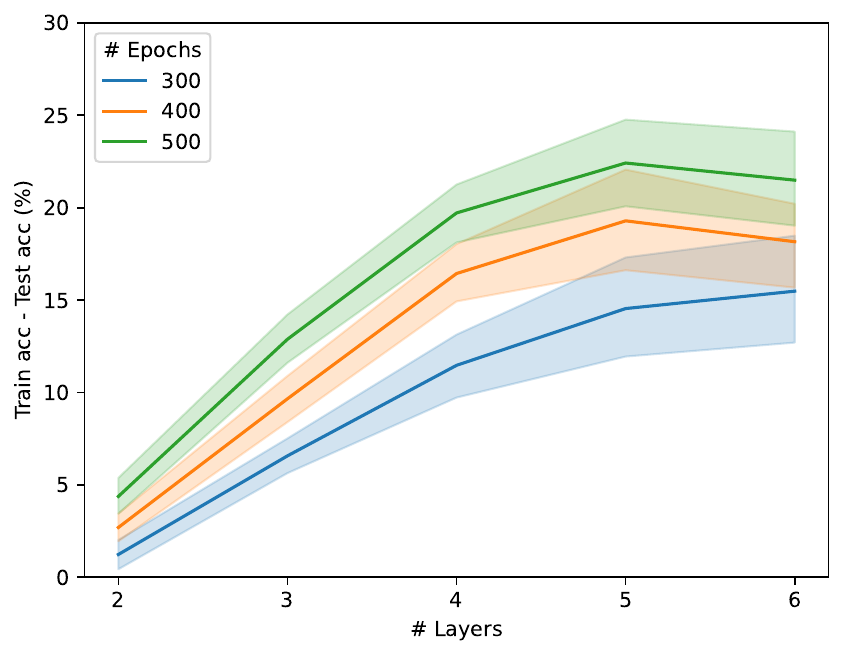}
    \end{subfigure}\caption{Results on the task \textbf{E1} for GNNs with activation function $\mathsf{tanh}$ over the dataset \textbf{PROTEINS}.}
\end{figure}

\begin{figure}[ht!]\label{fig:tanh_ptc_hidden}
    \begin{subfigure}{0.49\linewidth}
        \includegraphics[width=\textwidth]{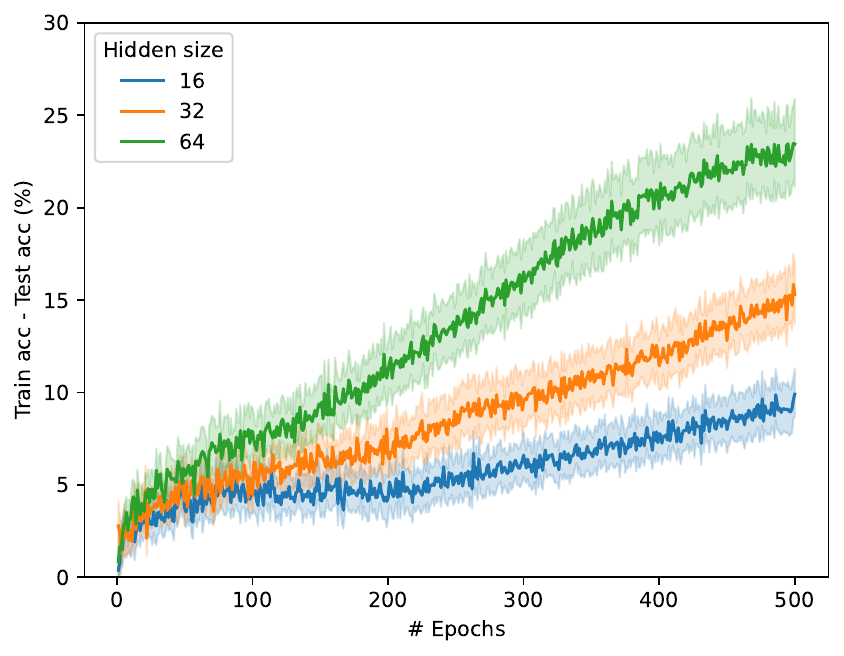}
    \end{subfigure}
    \begin{subfigure}{0.49\linewidth}
        \includegraphics[width=\textwidth]{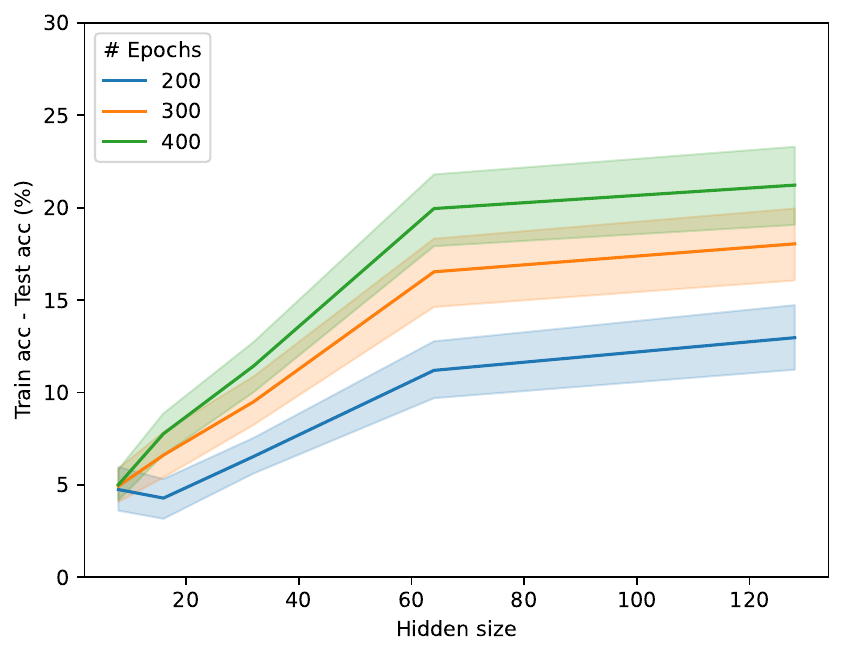}
    \end{subfigure}
    \begin{subfigure}{0.49\linewidth}
        \includegraphics[width=\textwidth]{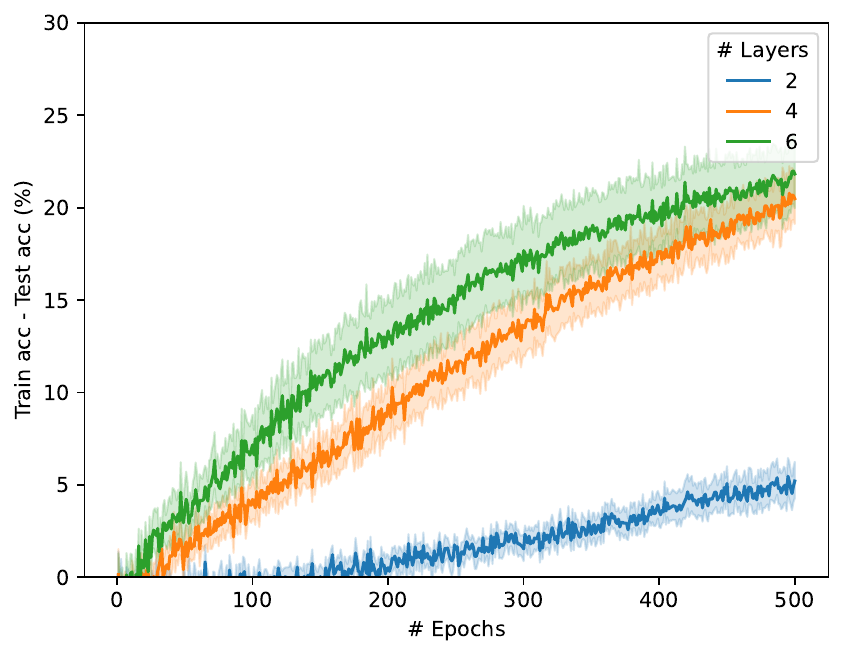}
    \end{subfigure}
    \begin{subfigure}{0.49\linewidth}
        \includegraphics[width=\textwidth]{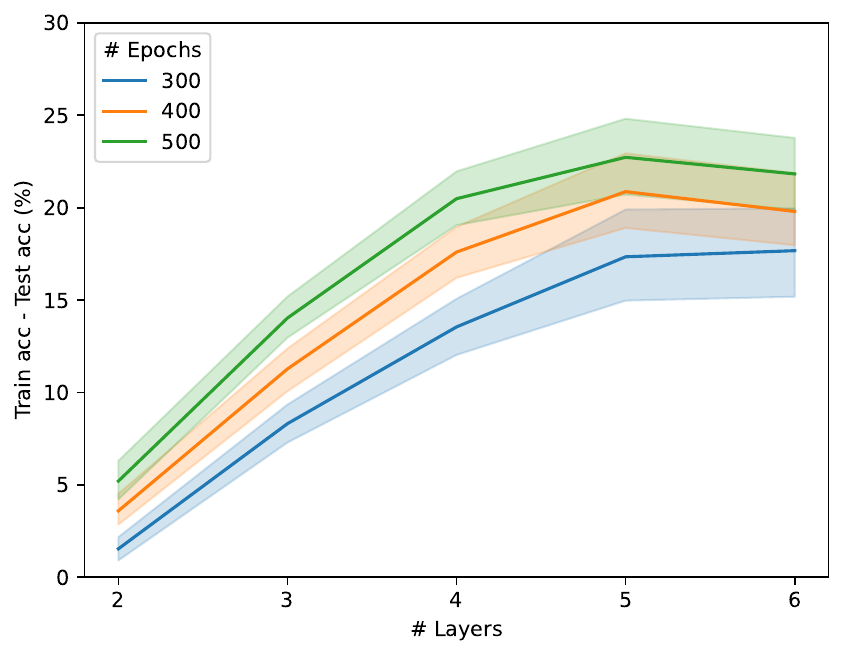}
    \end{subfigure}\caption{Results on the task \textbf{E1} for GNNs with activation function $\mathsf{tanh}$ over the dataset \textbf{PTC-MR}.}
\end{figure}

\end{document}